\documentclass[runningheads,a4paper]{llncs}
\usepackage{subfigure}
\usepackage{amssymb,amsmath}
\usepackage{multirow}
\usepackage{graphicx}
\usepackage{floatflt}
\usepackage{natbib,url}
\usepackage{hyperref}
\usepackage{multirow}
\usepackage{wrapfig}
\usepackage{bbm}

\urldef{\mailsa}\path|{alfred.hofmann, ursula.barth, ingrid.haas, frank.holzwarth,|
\urldef{\mailsb}\path|anna.kramer, leonie.kunz, christine.reiss, nicole.sator,|
\urldef{\mailsc}\path|erika.siebert-cole, peter.strasser, lncs}@springer.com|

\begin{document}

\mainmatter  % start of an individual contribution

\title{Spectral Clustering with Imbalanced Data}

\titlerunning{Spectral Clustering with Imbalanced Data}

\author{Jing Qian
\and Venkatesh Saligrama}
\authorrunning{Jing Qian and Venkatesh Saligrama}

\institute{Boston University, Department of Electrical and Computer Engineering,\\
8 Saint Mary's Street, Boston, MA 02215, USA} %\\
% \mailsa\\
% \mailsb\\
% \mailsc\\
% \url{http://www.springer.com/lncs}}


\toctitle{Lecture Notes in Computer Science}
\tocauthor{Authors' Instructions}
\maketitle

\begin{abstract}
Spectral clustering is sensitive to how graphs are constructed from data particularly when proximal and imbalanced clusters are present.
%Spectral clustering (SC) and graph-based semi-supervised learning (SSL) algorithms are sensitive to how graphs are constructed from data.
%In particular if the data has proximal and unbalanced clusters these algorithms can lead to poor performance on well-known graphs such as $k$-NN, full-RBF, $\epsilon$-graphs.
We show that Ratio-Cut (RCut) or normalized cut (NCut) objectives are not tailored to imbalanced data since they tend to emphasize cut sizes over cut values.
We propose a graph partitioning problem that seeks minimum cut partitions under minimum size constraints on partitions to deal with imbalanced data. Our approach parameterizes a family of graphs, by adaptively modulating node degrees on a fixed node set, to yield a set of parameter dependent cuts reflecting varying levels of imbalance. The solution to our problem is then obtained by optimizing over these parameters. We present rigorous limit cut analysis results to justify our approach. We demonstrate the superiority of our method through unsupervised and semi-supervised experiments on synthetic and real data sets.
\end{abstract}

%%%%%%%%%%%%%%%%%%%%%%%%%%%%%%%%%%%%%%%%%%
\section{Introduction}\label{sec:intro_motiv}
%%%%%%%%%%%%%%%%%%%%%%%%%%%%%%%%%%%%%%%%%%
Data with imbalanced clusters arises in many learning applications and has attracted much interest \cite{HeGarcia09}.
In this paper we focus on graph-based spectral methods for clustering and semi-supervised learning (SSL) tasks.
While model-based approaches \citep{Fraley02} may incorporate imbalancedness, they typically assume simple cluster shapes and need multiple restarts. In contrast non-parametric graph-based approaches do not have this issue and are able to capture complex shapes \citep{Ng01}.

In spectral methods, first a graph representing data is constructed and then spectral clustering(SC) \citep{Hagen92,Shi00} or SSL algorithms \citep{Zhu08,WanJebCha08} is applied on the resulting graph. Common graph construction methods include $\epsilon$-graph, fully-connected RBF-weighted(full-RBF) graph and $k$-nearest neighbor($k$-NN) graph. Of the three $k$-NN graphs appears to be most popular due to its relative robustness to outliers \citep{Zhu08,Luxburg07}.
Recently \citet{JebShc06} proposed $b$-matching graph which is supposed to eliminate some of the spurious edges of $k$-NN graph and lead to better performance.
%
%with size constraints the partitions are not necessarily low-density cuts, while our clustering goal here is to find natural partitions separated by density valleys -- clusters could be unbalanced but we do not know a priori how unbalanced they are.

%{\bf SRV: Briefly talk about related problems and related work. There are a few that have recognized unbalanced issues -- If you google unbalanced data clustering -- you will see a few refs. Also look at paper (Bin Yu and Belkin) I sent you. We need to point this out and say their goal is different -- it is really multi-clusterings.}

To the best of our knowledge, there do not exist systematic ways of adapting spectral methods to imbalanced data. We show that the poor performance of spectral methods on imbalanced data can be attributed to applying Ratio-Cut (RCut) or normalized cut (NCut) minimization objectives on traditional graphs, which sometimes tend to emphasize balanced partition size over small cut-values.

\noindent
{\bf Our Contributions:} \\
To deal with imbalanced data we propose partition constrained minimum cut problem (PCut). Size-constrained min-cut problems appear to be computationally intractable \citep{Galbiati11,Ji04}.
Instead we attempt to solve PCut on a parameterized family of cuts. To realize these cuts we parameterize a family of graphs over some parametric space $\lambda \in \Lambda$ and generate candidate cuts using spectral methods as a black-box. This requires a sufficiently rich graph parameterization capable of approximating varying degrees of imbalanced data. To this end we introduce a novel parameterization for graphs that involves adaptively modulating node degrees in varying proportions. We now solve PCut on a baseline graph over the candidate cuts generated on this parameterization. Fig.~\ref{f.framework} depicts our approach for binary clustering. Our limit cut analysis shows that our approach asymptotically does adapt to imbalanced and proximal clusters. We then demonstrate the superiority of our method through unsupervised clustering and semi-supervised learning experiments on synthetic and real data sets.
Note that we don't presume imbalancedness of the underlying density; our method significantly outperforms traditional approaches when the underlying clusters are imbalanced and proximal, while remaining competitive when they are balanced.

\begin{figure}[htbp]%{r}{.5\textwidth}
%\vspace{-20pt}
\begin{center}
\includegraphics[width=0.9\textwidth]{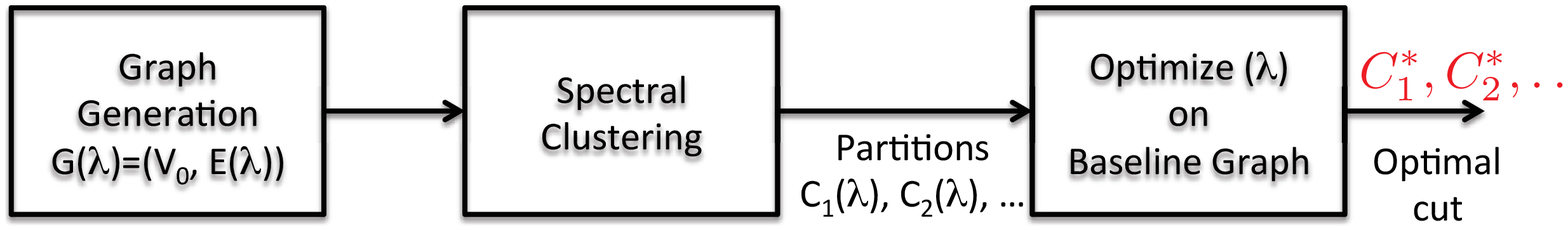}
\caption{\small Proposed Framework for Clustering on Imbalanced Data.}%\small Cut-ratio ($q$) vs unbalancedness ($y$). RCut value is smaller for balanced cuts than unbalanced low-density cuts whenever the cut-ratio is above the curve.}
\end{center}
\label{f.framework}
\vspace*{-15pt}
\end{figure}

%\begin{figure}[htbp]%{r}{.5\textwidth}
%\floatconts
%{f.framework}
%{ \caption{ Proposed Framework for Clustering on Imbalanced Data. } }
%{ \includegraphics[width=1\textwidth]{framework1.eps} }
%\end{figure}

\noindent
{\bf Related Work:} \\
Sensitivity of spectral methods to graph construction is well documented \citep{Luxburg07,Maier1,JebWanCha09}. \citet{Zelnik04} suggests an adaptive RBF parameter in full-RBF graphs to deal with imbalanced clusters. \citet{Nadler07} describes these drawbacks from a random walk perspective.
\citet{Buhler09,ShiBelkinYu09} also mention imbalanced clusters, but none of these works explicitly deal with imbalanced data. Besides, our approach is complementary to their schemes and can be used in conjunction.
Another related approach is size-constrained clustering \citep{ST97,Feige03,Andreev04,Hoppner08,ZhuWangLi10,Galbiati11}, which is shown to be NP-hard. \citet{Nagano11} proposes sub-modularity based schemes that work only for some special cases.
Besides, these works either impose exact cardinality constraints or upper bounds on the cluster sizes to look for balanced partitions.
While this is related, we seek minimum cuts with lower bounds on smallest-sized clusters.
Minimum cuts with lower bounds on cluster size naturally arises because we seek cuts at density valleys (accounted for by the min-cut objective) while rejecting singleton clusters and outliers (accounted for by cluster size constraint).
It is not hard to see that our problem is computationally no better than min-cut with upper bounds of size constraints. \footnote{In 2-way partition setting, min-cut with lower bounds is equivalent to min-cut with upper bounds, and is thus NP-hard. The multi-way partition problem generalizes 2-way setting. }

The organization of the paper is as follows.
In Section 2 we propose our partition constrainted min-cut (PCut) framework, illustrate some of the fundamental issues underlying poor performance of spectral methods on imbalanced data and explain how PCut can deal with it.
We describe the details of our PCut algorithm in Section 3, and explore the theoretical basis in Section 4.
In Section 5 we present experiments on synthetic and real data sets to show significant improvements in SC and SSL tasks.
Section 6 concludes the paper.
\section{Partition Constrained Min-Cut (PCut)}\label{sec:motiv}
%%%%%%%%%%%%%%%%%%%%%%%%%%%%%%%%%%%%%%%%%%
We first formalize PCut in the continuous setting. We assume that data is drawn from some unknown density $f(x)$, where $x \in \mathbb{R}^d$. %Let $S$ be a hypersurface partitioning $\mathbb{R}^d$:
We seek a hypersurface $S$ that partitions $\mathbb{R}^d$ into two subsets $D$ and $\bar D$ (with $D \cup \bar D=\mathbb{R}^d$) with non-trivial mass and passes through low-density regions:
%While there are many ways to formulate partitioning problems we formulate
%The goal is to find a hypersurface that passes through minimum density regions, namely,
\begin{equation} \label{e.optim}
S_0 = \mbox{arg}\min_{S} \int_S \psi(f(s)) ds,\,\,\, \mbox{subject to:} \,\,\,\min\{\mu(D),\mu(\bar D)\} \geq \delta > 0,
\end{equation}
where $\int_S$ stands for the $(d-1)$-dimensional integral, $\psi(\cdot)$ is some positive monotonic function, $\mu(A) = Prob\{x \in A\}$ is the probability measure, and $\delta$ is some positive constant. We describe imbalanced clusters as follows:
%Note that the optimal hypersurface $S_0$ may not necessarily lead to balanced partitions but that $S_0$ passes through ``density valleys.''

\begin{definition}
Data is said to be $\alpha$-imbalanced if $\alpha = \min\{\mu(D_0),\mu(\bar D_0)\} < 1/2$, where $(D_0,\bar D_0)$ is the optimal partitions obtained in Eq.(\ref{e.optim}).
\end{definition}

We next describe the problem in the finite data setting.  Let $G=(V,\,E,\,W)$ be a weighted undirected graph constructed from $n$ i.i.d. samples drawn from $f(x)$. Each node $v \in V$ is associated with a data sample. Edges are constructed using one of several graph construction techniques such as a $k$-NN graph. The weights on the edges are similarity measures such as RBF kernels that are based on Euclidean distances. % in some manner consistent with the underlying topology of the ambient space.
We denote by $S$ a cut that partitions $V$ into $C_S$ and $\bar C_S$. The cut-value associated with $S$ is:
\begin{eqnarray}\label{equ:cut}
Cut(C_S,\bar{C_S}) = \sum_{u\in C_S,v\in \bar{C_S},(u,v)\in E}w(u,v)
\end{eqnarray}
We pose the problem of partition size constrained minimum cut (PCut):
\begin{equation} \label{e.empopt}
\mbox{\bf PCut:} \,\,\,S_* = \mbox{arg} \min_{S} \left \{ Cut(C_S,\bar C_S) \mid \min\{|C_S|,\,|\bar C_S|\} \geq \delta |V| \right \} = S_*(C^*,\,\bar C^*).
\end{equation}

Eq.(\ref{e.empopt}) describes a binary partitioning problem but generalizes to arbitrary number of partitions. Note that without size constraints this problem is identical to min-cut criterion \citep{Stoer97}, which is well-known to be sensitive to outliers. This objective is closely related to the problem of graph partitioning with size constraints. Various versions of this problem are known to be NP-hard~\citep{Ji04}. Approximations to such partitioning problems have been developed~\citep{Andreev04} but appear to be overly conservative. More importantly these papers~\citep{Andreev04,Hoppner08,ZhuWangLi10} either focus on balanced partitions or cuts with exact size constraints. In contrast our objective here is to identify natural low-density cuts that are not too small(i.e. with lower bounds on smallest sized cluster).
We here employ SC as a black-box to generate candidate cuts on a suitably parameterized family of graphs. Eq.(\ref{e.empopt}) is then optimized over these candidate cuts.

%This formulation can deal with unbalanced clusters in this sense:
%\begin{prop}
%
%Partitions $C^*$ or  $\bar C^*$ cannot be themselves made up of \emph{disconnected} sub-clusters. To see this we suppose $G$ is a connected graph and $(C, \bar C)$ an arbitrary partition satisfying the $\delta$-size constraint. Suppose $C =  C_1 \cup C_2$ with $C_1 = \max \{|C_1|,\, |C_2| \} \geq \delta |V|$ and $C_1$ is disconnected from $C_2$ on the induced sub-graph associated with $C$. Then the modified partition $(C_1, \bar C \cup C_2)$ is preferred since it has a lower cut-value and satisfies the $\delta$-size constraint. Second the cut $S$ passes through a ``\emph{density valley},'' namely, $Cut(C_1, \bar C^* \cup C_2) \geq Cut(C^*, \bar C^*)$ for every $C_1, C_2$ such that $C^*=(C_1 \cup C_2)$ and $\max \{|C_1|,\, |C_2| \} \geq \delta |V|$. %partitions.
%\end{prop}
%

%
%
%
%We now focus on finite sample objective mirroring the continuous objective of Eq.(\ref{e.optim}).
%
%\end{remark}
% We assume that the cut $S_*$ results in $\left(C_*,\,\bar C_*\right)$.
%
%%%%%%%%%%%%%%%%%%%%%%%%%%%%%%%%%%%%%%%%%%
\subsection{RCut, NCut and PCut}\label{subsec:alg}
%%%%%%%%%%%%%%%%%%%%%%%%%%%%%%%%%%%%%%%%%%
%Up to our knowledge there does not seem to be any algorithms that directly solve Eq.(\ref{e.empopt}).
%Existing graph partitioning algorithms aim to minimize various objectives on the graph.
The well-known spectral clustering algorithms attempt to minimize RCut or NCut:
\begin{eqnarray}\label{equ:ratiocut}
    \min_S:\,\,Cut(C_S,\bar{C_S}) \left( \frac{size(V)}{size(C_S)} + \frac{size(V)}{size(\bar{C_S})} \right),
    %\left(\frac{1}{|C_S|}+\frac{1}{|\bar{C_S}|}\right),
 \end{eqnarray}
where $size(C)=|C|$ for RCut and $size(C)=\sum_{u\in C,v\in V}w(u,v)$ for NCut. Both objectives seek to trade-off low cut-values against cut size.
While robust to outliers, minimizing RCut(NCut) can lead to poor performance when data is imbalanced (i.e. with small $\alpha$ of Def.1).
To see this, we define cut-ratio $q \in [0,\,1]$, and imbalance coefficient $y \in [0,\,0.5]$, for some graph $G=(V,E,W)$:
$$
q = {Cut(C^*,\bar{C^*}) \over Cut(C_B,\bar{C_B})};\,\,\,y= \frac{\min \{size(C^*),\,size(\bar{C^*})\}}{size(C^*)+size(\bar{C^*})}.
$$
% where $size(C)=|C|$ for RCut and $size(C)=vol(C)$ for NCut.
where $(C^*,\bar C^*)$ corresponds to optimal PCut and $S_B(C_B,\bar C_B)$ is any balanced partition with $size(C_B)=size(\bar C_B)$. We analyze the limit-cut behavior of $k$-NN, $\epsilon$-graph and RBF graph to build intuition. For %$k$-NN, full-RBF and $\epsilon$-graph with
properly chosen $k_n$, $\sigma_n$ and $\epsilon_n$ \citep{Maier1,Narayanan06}, as sample size $n\rightarrow \infty$, we get:
%\vspace{-0.15in}
%{\small
%\begin{equation} \label{e.limval}
% q \stackrel{n \rightarrow \infty}{\longrightarrow} {\int_{S_0} f^{\gamma}(x)dx \over \int_{S_B} f^{\gamma}(x)dx},\,\,\,\,\, y \stackrel{n \rightarrow \infty}{\longrightarrow} \min \{\mu(D_0),\,\mu(\bar D_0)\}
\begin{equation} \label{e.lcutqy}
q \longrightarrow {\int_{S_0} f^{\gamma}(x)dx \over \int_{S_B} f^{\gamma}(x)dx},\,\,\,\,\, y \longrightarrow \min \{\mu(D_0),\,\mu(\bar D_0)\} = \alpha
\end{equation}
\noindent
where $\gamma<1$ for $k$-NN and $\gamma \in [1,2]$ for $\epsilon$-graph and full-RBF graphs. Asymptotically we can say:

{\bf (1)} While cut-ratio $q$ varies with graph construction, the imbalance coefficient $y$ is invariant. In particular we can expect $q$ for $k$-NN to be larger relative to $q$ for full-RBF and $\epsilon$-graph since $\gamma < 1$.

{\bf (2)} PCut appears to lead to similar results for all graph constructions. This %\footnote{\cite{Maier1} shows different behaviors for different graphs when there are two density valleys, in which case both valleys are meaningful results to Eq.(\ref{e.optim}).} and converges to the optimal solution in Eq.(\ref{e.optim}).
follows directly from limit-cut behavior and the limiting independence of $y$ to graph construction. %(see {\bf (1)}).

{\bf (3)} Optimal (limiting) RCut/NCut depends on graph construction. Indeed, for any partition $(C, \bar C)$ $$RCut(C,\bar C) = Cut(C, \bar C) \left ({1 \over |C|} +  {1 \over |\bar C|} \right ) \implies  {RCut(C^*,\bar C^*) \over RCut(C_B,\bar C_B)} = {q \over 4 y(1-y)}$$
A similar expression holds for NCut with appropriate modifications. Because $q$ varies for different graphs but $y$ does not, the ratio $(q/4y(1-y))$ depends on graph construction.
So it is plausible that for some constructions RCut (or NCut) value satisfies $q > 4y(1-y)$ while for others $q < 4y(1-y)$. In the former case RCut/NCut will favor a balanced cut over ``density valley'' cut $(C^*,\bar C^*)$ and vice versa if the latter is true. Fig.~\ref{fig:qy} depicts this point.
%Furthermore, since $y$ is invariant to graph construction as opposed to $q$ we see from {\bf (1)} that different graph constructions lead to different cuts when RCut/NCut is minimized
That this possibility is real is illustrated in Fig.~\ref{fig:2g_graph} for a Gaussian mixture. There RBF $k$-NN with large $\sigma$ favors balanced cut and does not for small $\sigma$. %while it does not favor balanced cuts for typically has larger $q$ values it is more likely to favor balanced cuts. Again Fig~\ref{fig:2g_graph} shows this behavior for RBF $k$-NN for large of RBF $\sigma$ values.

\begin{figure}[htbp]%{r}{.5\textwidth}
%\vspace{-20pt}
\begin{center}
\includegraphics[width=0.5\textwidth]{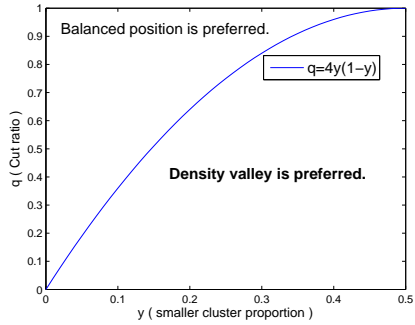}
\caption{\small Cut-ratio ($q$) vs imbalance ($y$). RCut value is smaller for balanced cuts than imbalanced low-density cuts for cut-ratios above the curve.}
\end{center}
\label{fig:qy}
%\vspace*{-15pt}
\end{figure}

%\begin{figure}[htbp]%{r}{.5\textwidth}
%\floatconts
%{fig:qy}
%{ \caption{ Cut-ratio ($q$) vs imbalance ($y$). RCut value is smaller for balanced cuts than imbalanced low-density cuts for cut-ratios above the curve. } }
%{ \includegraphics[width=.5\textwidth]{qy_plot.eps} }
%\end{figure}

{\bf (4)} We can loosely say that if data is imbalanced and sufficiently proximal (close clusters) then asymptotically $k$-NN, full-RBF and $\epsilon$-graph can all fail when RCut is minimized. To see this consider an imbalanced mixture of two Gaussians similar to Fig.~\ref{fig:2g_graph}. By suitably choosing the means and variances we can construct sufficiently proximal clusters with same imbalance but relatively large $q$ values. This is because $f(x)$ will be relatively large even at density valleys for proximal clusters. Our statement now follows from {\bf (3)}.
%
%\begin{wrapfigure}{r}{.45\textwidth}
%\centering
%\includegraphics[width=.45\textwidth]{qy_plot.eps}
%\caption{\small Cut-ratio ($q$) vs imbalance ($y$). RCut value is smaller for balanced cuts than imbalanced low-density cuts for cut-ratios above the curve.}
%\label{fig:qy}
%\end{wrapfigure}
%
%

In summary, we have learnt that optimal RCut/NCut depend on graph construction and can fail for imbalanced proximal clusters for $k$-NN, $\epsilon$-graph, full-RBF constructions on same data. PCut is computationally intractable but asymptotically invariant to graph construction and picks the right answer. Since SC is relaxed variant of optimal RCut/NCut we can expect it to have similar behavior relative to optimal RCut/NCut. Nevertheless, SC is computationally tractable. This motivates the following section.
%%%%%%%%%%%%%%%%%%%%%%%%%%%%%%%%%%%%%%%%%%
\subsection{Using Spectral Clustering for PCut} \label{subsec:sol}%How to handle imbalanced data with RCuts/NCuts?}
%%%%%%%%%%%%%%%%%%%%%%%%%%%%%%%%%%%%%%%%%%
%We expect cut-ratio $q$ to be large whenever proximal clusters are present since $f(x)$ is relatively large. So this implies that if the unbalancedness is small and clusters are proximal ratio-cut objective would fail. This point can also be seen from limit cut analysis \citep{Maier1,Narayanan06} of K-NN, $\epsilon$-neighborhood and RBF graphs.
%
 %Thus from Proposition~1 it follows that ratio-cut with these constructions would fail to cluster the unbalanced nodes.
Fortunately, limit-cut analysis does not explain the behavior of RCut/NCut. We refer to Fig.\ref{fig:2g_graph} for a different perspective. Here we plot RCut values as a function of cut positions. We would like the minimum value of RCut to be achieved at $x_1=1$ since this solves PCut. Note that RBF $k$-NN for different values of RBF parameters exhibits entirely different behaviors in (b). For larger $\sigma$ (green) we see RBF $k$-NN has a balanced cut while for small $\sigma$ (blue) balanced position is not a minima (although meaningless cuts at boundaries may result which violate the constraints in Eq.(\ref{e.optim})). A similar story emerges for full-RBF and $\epsilon$-graph. As an aside, which we explain later, our approach, Rank Modulated Degree (RMD) construction, has a minima (black) at $x_1=1$ and is robust to outliers at boundary regions.

\begin{figure*}[h]%[!tb]
\begin{centering}
\begin{minipage}[t]{.32\textwidth}
\includegraphics[width = 1.15\textwidth]{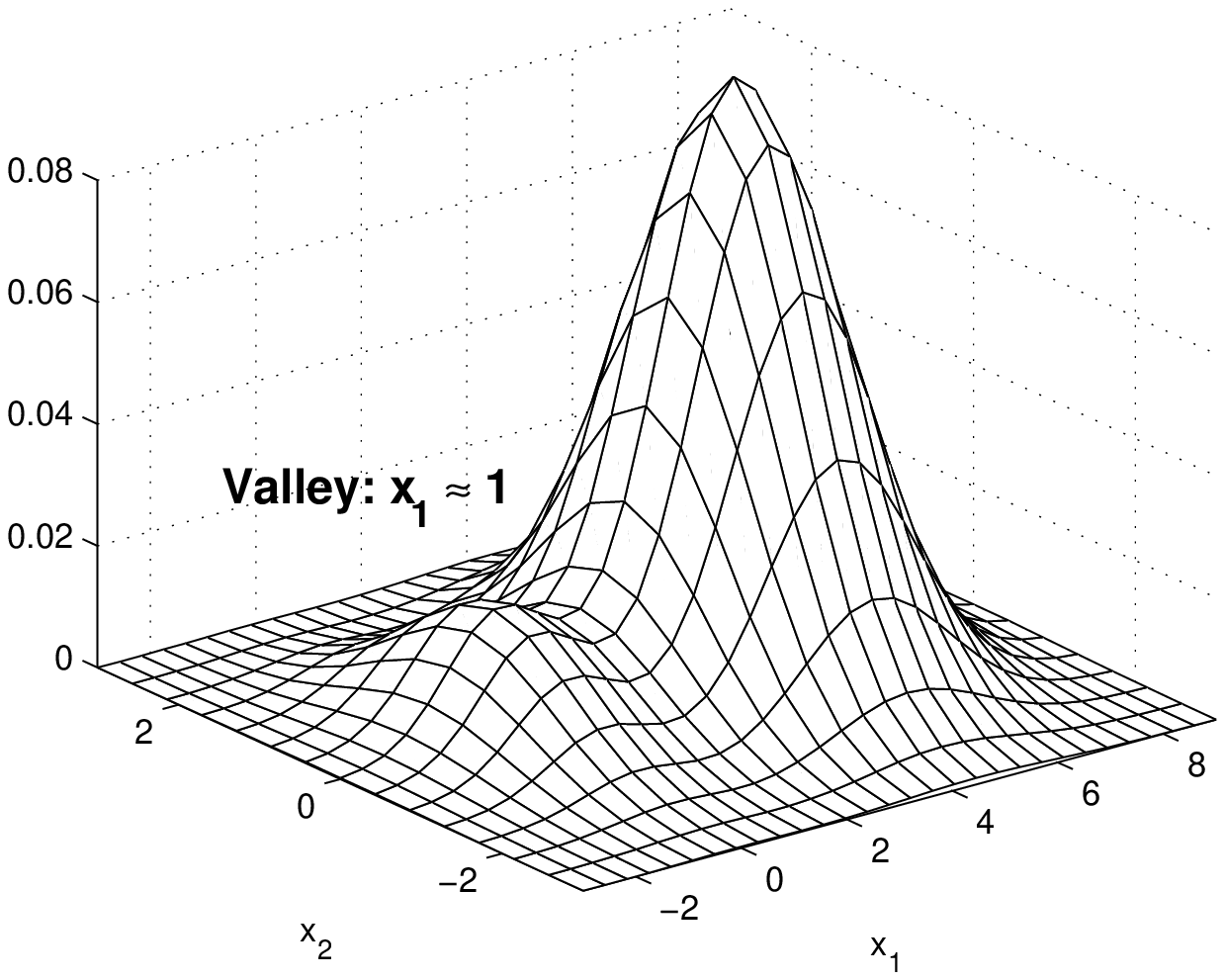}
\makebox[5 cm]{\small (a) Proximal Imbalanced Clusters}
\end{minipage}
\begin{minipage}[t]{.32\textwidth}
\includegraphics[width = 1.13\textwidth]{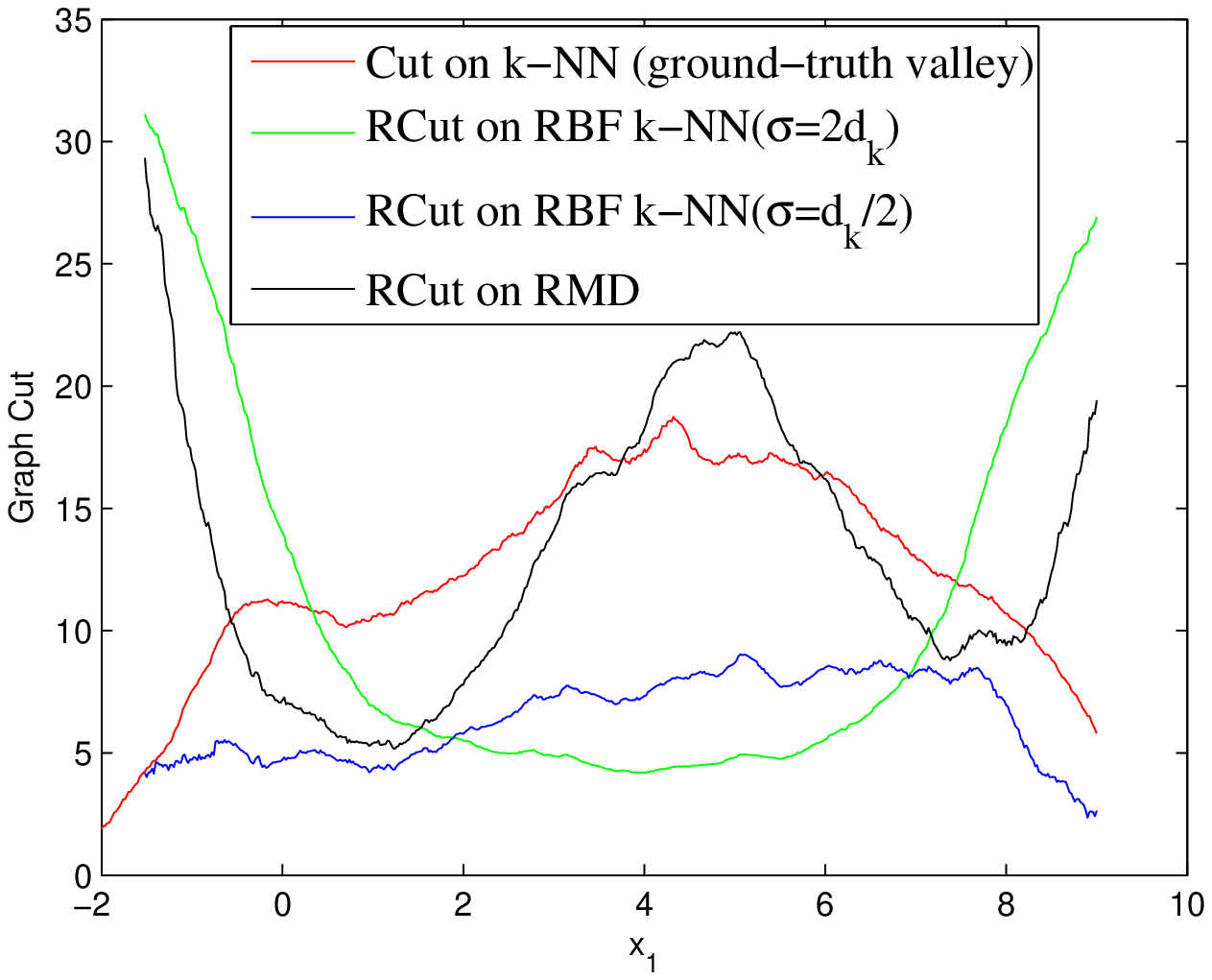}
\makebox[5 cm]{\small (b) $k$-NN and RMD }
\end{minipage}
\begin{minipage}[t]{.32\textwidth}
\includegraphics[width = 1.15\textwidth]{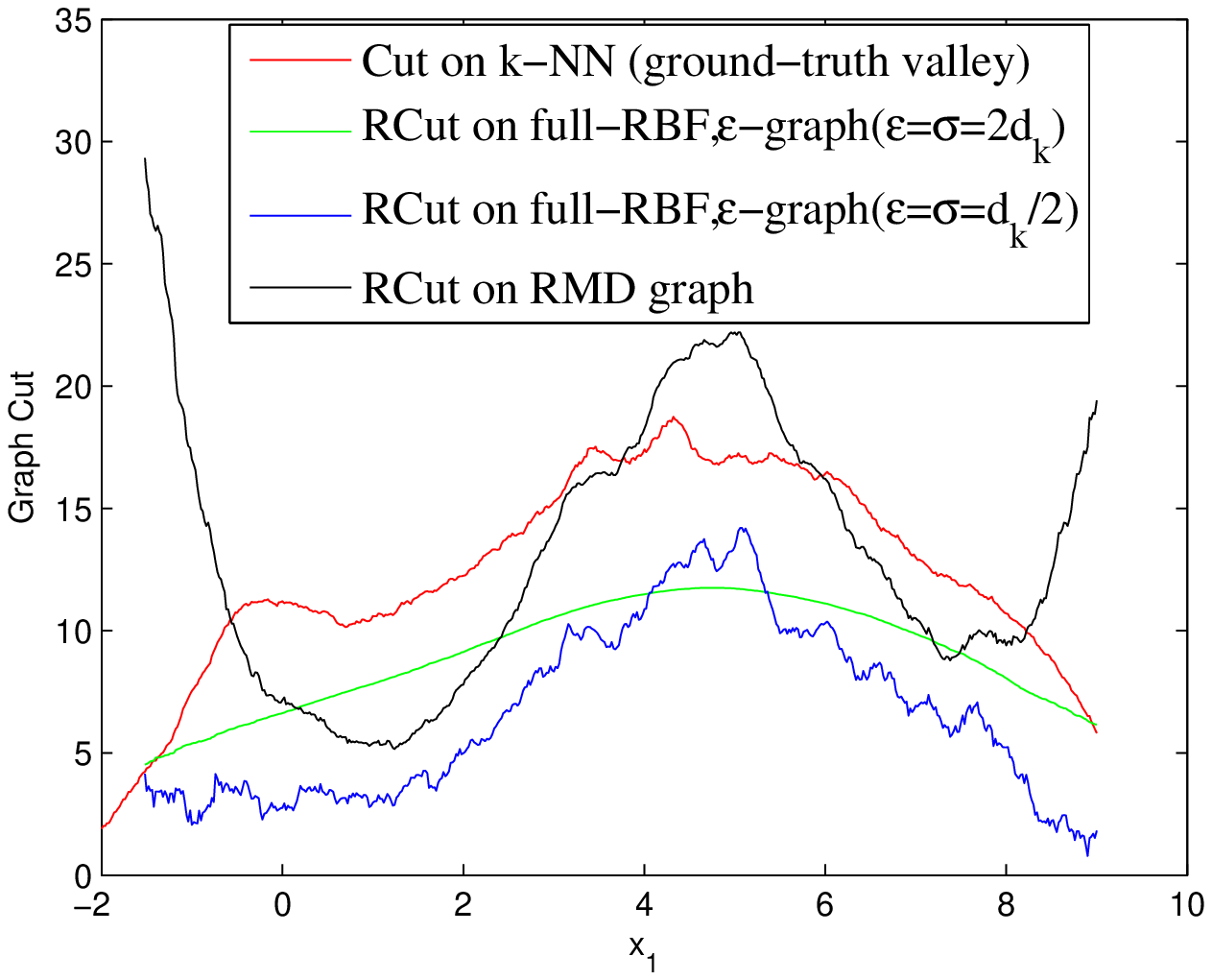}
\makebox[5 cm]{\small (c) $\epsilon$-RBF and RMD }
\end{minipage}
%\begin{minipage}[t]{.32\textwidth}
%\includegraphics[width = 1.15\textwidth]{2g_knn_new.eps}
%\makebox[5.5 cm]{\small (d) $k$-NN}
%\end{minipage}
%\begin{minipage}[t]{.32\textwidth}
%\includegraphics[width = 1.15\textwidth]{2g_full_new.eps}
%\makebox[5.5 cm]{\small (e) full-RBF and $\epsilon$-graph}
%\end{minipage}
%\begin{minipage}[t]{.32\textwidth}
%\includegraphics[width = 1.15\textwidth]{2g_rmd_new.eps}
%\makebox[5.5 cm]{\small (f) RMD(our method)}
%\end{minipage}
\caption{ Imbalanced mixture of two Gaussians with mixture proportions $0.85$ and $0.15$; means are $[4.5;0],\,\,[0;0]$; covariances $diag(2,1),\,I$. Optimal cut $S_0$ of Eq.~\ref{e.optim} is the hyperplane $x_1=1$ and a balanced cut is a line passing through $x_1=4$.
%Cut and RCut values for $k$-NN, $\epsilon$-graph, full-RBF and our graph(RMD) are plotted.
Figures in (b),(c) are averaged over 20 Monte Carlo runs with $n=1000$; $\sigma$ the RBF parameter, $d_k$ the average $k$-NN distance with $k=30$. Black, Blue \& Green curves are for RCut. Red curve is the cut-value. %In (c) full-RBF and $\epsilon$-graph have similar behaviors.
All curves are re-scaled for illustration. %demonstration. suitably normalized for the purpose of illustration.
%For (f) unweighted RMD graph with $l=30, \lambda=0.4$; for (d) unweighted $k$-NN; for (e) $\epsilon=\sigma=d_k$. Notice the ground-truth valley of Cut curve(red) in (b),(c). (b) shows large $\sigma$ in the $k$-NN graph results in smoothing of cut-values and the minimum RCut is not at the density valley. (c),(e) show smaller $\epsilon,\sigma$ have pronounced sensitivity to outliers (RCut curves (green/blue in (c)) go down near boundaries), while large $\epsilon,\sigma$ smoothen the RCut value.
}
\label{fig:2g_graph}
\end{centering}
% \vspace{-10pt}
\end{figure*}

\emph{The preceding example suggests that the position (or hyper surface) where Rcut/Ncut achieves its minimum value depends on the choice of graph parameter} as seen from the dramatic difference for RBF $k$ NN for different choices of $\sigma$. Notice in Fig.~\ref{fig:2g_graph}(b) the performance of RCut on RBF $k$-NN graph is sensitive to change in RBF parameter. This dramatic change when $\sigma$ is increased 4 fold can be explained through the cut-ratio. Large values of $\sigma$ tend to put equal weights on all the neighbors of each node. Small values of $\sigma$ tend to be non-uniform with some edges having small weights. Furthermore, for smaller values of $\sigma$ edges in highly dense regions tend to have uniformly larger weights while in low-density regions the edges tend to have smaller and non-uniform weights.

This discussion suggests the possibility of controlling cut-ratio $q$ through graph parameters while not impacting $y$ (since it is invariant to different $\sigma$ choices). This key insight leads us to the following idea for PCut:

\noindent
%\begin{itemize}
%\item [(A)]
{\bf (A) Parameter Optimization:}  Generate several candidate optimal RCuts/NCuts as a function of graph parameters. Pose PCut over these candidate cuts rather than arbitrary cuts as in Eq.(\ref{e.empopt}). Thus PCut is now parameterized over graph parameters.
%\item [(B)]

\noindent
{\bf (B) Graph Parameterization:} If the graph parameters are not sufficiently rich that allow for adaptation to imbalanced or proximal cuts ({\bf (A)}) would be useless. Therefore, we want graph parameterizations that allow sufficient flexibility so that the posed optimization problem is successful for a broad range of imbalanced and proximal data.
%\end{itemize}

%\noindent
%{\bf RBF $k$-NN graphs:}
%Let us first motivate the second objective. Notice in Fig.~\ref{fig:2g_graph}(b) the performance of RCut on RBF-$k$-NN graph is sensitive to change in RBF parameter. This dramatic change when RBF parameter, $\sigma$, is increased 4 fold can be explained through the cut-ratio. Large values of $\sigma$ tend to put equal weights on all the neighbors of each node. Small values of $\sigma$ tend to be non-uniform with some edges having small weights. Furthermore, for smaller values of $\sigma$ edges in highly dense regions tend to have uniformly larger weights while in low-density regions the edges tend to have smaller and non-uniform weights. Consequently, for the same imbalance, $y$, the cut-ratio $q$ is typically smaller for smaller $\sigma$ and balanced cuts are not preferred.

We first consider the second objective. We have found in our experiments (see Sec. 5) that the parametrization based on RBF $k$-NN graphs is not sufficiently rich to account for varying levels of imbalanced and proximal data. To induce even more flexibility we introduce a new parameterization:

\noindent
{\bf Rank Modulated Degree (RMD) Graphs:} \\
We introduce RMD graphs that are a richer parameterization of graphs that allow for more control over $q$ and offers sufficient flexibility in dealing with a wide range of imbalanced and proximal data. We use the insight gained above about what happens to cut ratio as $\sigma$ is varied in RBF-$k$-NN in the low-density and high-density regions.   We consider two equivalent ways of accomplishing this task:

\noindent ({\bf I}) Adaptively modulate the node-degree on a baseline $k$-NN graph.

\noindent ({\bf II}) Modulate the neighborhood size for each node on a baseline $\epsilon$-graph.

%\begin{figure}[h]%{r}{.5\textwidth}
%%\vspace{-20pt}
%\begin{center}
%\includegraphics[width=0.8\textwidth]{framework1.eps}
%\caption{\small Proposed Optimization Framework for Clustering.}%\small Cut-ratio ($q$) vs unbalancedness ($y$). RCut value is smaller for balanced cuts than unbalanced low-density cuts whenever the cut-ratio is above the curve.}
%\end{center}
%\label{f.framework}
%\vspace*{-15pt}
%\end{figure}
%
%Both strategies are somewhat equivalent.
We adopt ({\bf I}) since we can easily ensure graph connectivity. \emph{Our idea is a parameterization that selectively removes edges in low density regions and adds edges in high-density regions}. This modulation scheme is based on rankings of data samples, which reflect the relative density. Our RMD scheme can adapt to varying levels of imbalanced data because we can reduce $q$ through modulation while keeping $y$ fixed. The main remaining issue is to reliably identify high/low density nodes for which we use a novel ranking scheme.

We are now left to pose PCut over graph parameters or candidate cuts. We describe this in detail in the following section. We construct a universal baseline graph for the purpose of comparison among different cuts and to pick the cut that solves Eq.(\ref{e.empopt}). These different cuts are obtained by means of SC and are parameterized by graph construction parameters. PCut is now solved on the baseline graph over candidate cuts realized from SC.
%The optimization problem on the baseline graph will select it.
Fig.\ref{f.framework} illustrates our framework for clustering.

%%%%%%%%%%%%%%%%%%%%%%%%%%%%%%%%%%%%%%%%%%
\section{Our Algorithm}\label{sec:RMD_idea}
%%%%%%%%%%%%%%%%%%%%%%%%%%%%%%%%%%%%%%%%%%
%We propose to modulate the node degrees through a parameterized scheme, which is based on ranking of data samples. We call the resulting graph the Rank-Modulated Degree(RMD) graph. Our graph is able to find the valley cut while being robust to outliers. Moreover we provide a model selection step making our graph adaptable to data with varying levels of unbalancedness.
Given $n$ data samples, our task is unsupervised clustering or SSL, assuming the number of clusters/classes $K$ is known. We start with a baseline $k_0$-NN graph $G_0=(V,E_0)$ built on these samples with $k_0$ large enough to ensure graph connectivity.
%Each node $v \in V$ is associated with a data sample $x_v$ and the graph is formed by connecting the $k_0$ nearest neighbors with respect to the Euclidean distance between data samples associated the nodes. We choose a sufficiently large $k_0$ to ensure graph connectivity.
Main steps of our PCut framework are as follows.
% where $W_0$ is the RBF similarity for each edge with parameter $\sigma$
%Our RMD graph-based learning framework has the following steps:
%

\begin{table*}[h]
% \caption{Error rate performance of GRF and GTAM for imbalanced real data sets. Our method performs significantly better than other methods.}
%\begin{center}
\begin{tabular}{l}
  \hline
  \textbf{ Main Algorithm:\, RMD Graph-based PCut } \\
  \hline
  1. Compute the rank $R(x_i)$ of each sample $x_i,i=1,...,n$; \\
  2. For different configurations of parameters, \\
    \,\,\, a. Construct the parametric RMD graph; \\
    \,\,\, b. Apply spectral methods to obtain a $K$-partition on the current RMD graph;  \\
  3. Among various partition results, pick the ``best" (evaluated on baseline $G_0$). \\
  \hline
\end{tabular}
%\end{center}
\label{tab:main_alg}
\end{table*}

\noindent
{\bf (1) Rank Computation:} \\
We compute the rank $R(v)$ of every node $v$ as follows:
\begin{eqnarray}\label{eq:grank}
  R(x_v) = \frac{1}{n}\sum_{w \in V} \mathbb{I}_{\{\eta(x_v)\leq \eta(x_w)\}}
\end{eqnarray}
where $\mathbb{I}$ denotes the indicator function, and $\eta(x_v)$ is some statistic reflecting the relative density at node $v$. Since $f$ is unknown, we choose average nearest neighbor distance as a surrogate for $\eta$. To this end let $N(v)$ be the set of all neighbors for node $v \in V$ on the baseline graph, and we let: %need to employ some surrogate statistic. While many choices are possible we use average nearest-neighbor distances in this paper.
\begin{eqnarray}\label{eq:grank}
  \eta(x_v)=\frac{1}{|N(v)|}\sum_{w \in N(v)} \|x_v - x_w\|.
\end{eqnarray}
The ranks $R(x_v) \in [0,\,1]$, are relative orderings of samples and are uniformly distributed.
$R(x_v)$ indicates whether a node $v$ lies near density valleys or high-density areas, as shown in Fig.\ref{f.pdf_rank}.
% The ranks  and are found to be insensitive to the choice of the baseline graph parameter $k_0$.

\begin{figure}[htbp]%{r}{.5\textwidth}
%\vspace{-20pt}
\begin{center}
\includegraphics[width=0.5\textwidth]{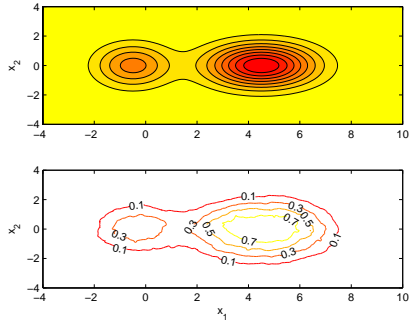}
\caption{\small Density level sets \& rank estimates.}
\end{center}
\label{f.pdf_rank}
%\vspace*{-15pt}
\end{figure}

%\begin{figure}[htbp]%{r}{.5\textwidth}
%\floatconts
%{f.pdf_rank}
%{ \caption{ Density level sets \& rank estimates. } }
%{ \includegraphics[width=0.45\textwidth]{gauss2D_pdf_rank.eps} }
%\end{figure}
%\begin{wrapfigure}{r}{.5\textwidth}
%\vspace{-20pt}
%\centering
%\includegraphics[width=.33\textwidth]{gauss2D_pdf_rank.eps}
%\vspace{-10pt}
%\caption{\small Density level sets \& rank estimates.}
%\label{fig:rwcont}
%\vspace{-10pt}
%\end{wrapfigure}

%Such statistics have been employed for high-dimensional anomaly detection \citep{Zhao09,Zhao12}. Details are described in Sec.\ref{subsec:rank}. The rank is a normalized ordering of all points based on $G$, ranges in $[0,1]$, and indicates how extreme $x$ is among all points.

\noindent
{\bf (2) Parameterized family of graphs:} \\
We consider three parameters, $\lambda \in [0,\,1]$, $k$ for $k$-NN and $\sigma$ for RBF similarity. These are then suitably discretized. We generate a weighted graph $G(\lambda,k,\sigma)=(V, E(\lambda,k,\sigma),W(\lambda,k,\sigma))$ on the same node set as the baseline graph but with different edge sets. For each node $v \in V$ we construct edges with $k_\lambda(v)$ nearest neighbors: %The degree deg($x$) for node $x$ is modulated as follows:
\begin{eqnarray}\label{eq:degree}
  k_{\lambda}(v) = k(\lambda+2(1-\lambda)R(x_v)),
\end{eqnarray}
%where $k$ is another parameter.
%where $\lambda \in (0,1]$ parameterizes the family of RMD graphs.
%$k$ is the average degree. We discretize $\lambda$ in $\{0.2,0.4,0.6,0.8,1\}$ in experiments. It is not difficult to see that $R(x)$ converges (in distribution) to a uniform measure on the unit interval regardless of the underlying density $f(\cdot)$.
%Note that since $R(x_v)$ is uniformly distributed, the average degree across all samples is $k$. %The weights on the edges are RBF weights with parameter $\sigma$.
%The minimum degree $\lambda k$ can be used to ensure a connected graph when necessary.
%Note that we also vary $k$,$\sigma$ in our experiments for a more thorough demonstration.
%Furthermore, the above modulation scheme can be thought of as modulating the degree of each node around a nominal value equal to $k$. The remaining issue is to optimize over the scalar parameter $\lambda$, which is described in Step (4). %We choose $\lambda$ through an optimization step (see below Step (4)).
This generates RMD parameterization. For other parameterizations such as RBF $k$-NN we let $\lambda=1$ and vary only $k,\,\sigma$.

\noindent
{\bf (3) Parameterized family of cuts:} \\
From $G(\lambda,k,\sigma)$ we generate a family of K partitions $C_1(\lambda,k,\sigma), C_2(\lambda,k,\sigma), \ldots, C_K(\lambda,k,\sigma)$. These cuts are generated based on the eventual learning objective. For instance, if K-clustering is the eventual goal these K-cuts are generated using SC. For SSL we use RCut-based Gaussian Random Fields(GRF) and NCut-based Graph Transduction via Alternating Minimization(GTAM) to generate cuts. These algorithms all involve minimizing RCut(NCut) as the main objective (SC) or some smoothness regularizer (GRF,GTAM). For details about these algorithms readers are referred to references \citep{Zhu08,WanJebCha08,Luxburg07,Chung96}.

\noindent
{\bf (4) Parameter Optimization:} \\
The final step is to solve Eq.(\ref{e.empopt}) on the baseline graph $G_0$.
We assume prior knowledge that the smallest cluster is at least of size $\delta n$.
The $K$-partitions obtained from step (3) are now parameterized: $\left(C_1(\lambda,k,\sigma),...,C_K(\lambda,k,\sigma)\right)$.
We optimize over these parameters to obtain the minimum cut partition (lowest density valley) on $G_0$.
%such that the smallest cluster is no smaller than some threshold $\delta$:
\begin{eqnarray}\label{eq:selection}
  & \min_{\lambda,k,\sigma}\{Cut_0\left(C_1,...,C_K\right)=\sum^K_{i=1}Cut_0(C_i,\bar{C}_i)\} \\
\nonumber
  & s.t. ~~\min\{|C_1(\lambda,k,\sigma)|,...,|C_K(\lambda,k,\sigma)|\}\geq \delta n
\end{eqnarray}
%$\delta$ sets the threshold of minimum cluster size, i.e.
$Cut_0(\cdot)$ denotes evaluating cut values on the baseline graph $G_0$. Partitions with clusters smaller than $\delta n$ are discarded.
%$Cut_0\left(\cdot\right)$ represents the Cut values of different partitions are evaluated on a same reference $k_0$-NN graph to pick the min-cut partition.
%{\bf This step exactly aims at the optimal criterion of Eq.(\ref{e.empopt})}.
%Note that whatever RCut/NCut is used, for the above size constraint we just consider the number of points within the clusters.
% Also notice that for a fair comparison, the Cut values of different partitions are computed on the same graph to choose a min-cut partition.

%\noindent
%\begin{tabular}{lll}
%  \hline
%  % after \\: \hline or \cline{col1-col2} \cline{col3-col4} ...
%  \noindent\textbf{Algorithm 1: RMD Graph-based Learning:} \\
%{\bf Input}: $n$ data samples $\{x_1,\ldots,x_n\}$ (partially labeled for SSL), number of clusters/classes $K$, \\
%            smallest cluster/class size threshold $\delta$.\\
%{\bf Steps}:\\
%  1. Compute ranks of samples based on Eq.(\ref{eq:grank}). \\
%  2. For different $\lambda,k,\sigma$, do: \\
%  \indent a. Construct the RMD graph based on Eq.(\ref{eq:degree}); \\
%  \indent b. Apply graph-based learning algorithms on the current RMD graph to get $K$ clusters. \\
%  3. Compute Cut values of different partitions from step 2 on the $k_0$-NN graph. Pick the partition \\
%  with the smallest Cut value based on Eq.(\ref{eq:selection}).\\
%{\bf Output}: the selected $K$-partition. \\
%  \hline
%\end{tabular}

\noindent
{\bf Remark:} \\
1. Although step (4) suggests a grid search over several parameters, it turns out that other parameters such as $k,\sigma$ do not play an important role as $\lambda$. Indeed, the experiment section will show that while step (4) can select appropriate $k,\sigma$, it is by searching over $\lambda$ that adapts spectral methods to data with varying levels of imbalancedness (also see Thm.\ref{part2}).\\
2. Our framework uses existing spectral algorithms and so can be combined with other graph-based partitioning algorithms to improve performance for imbalanced data, such as 1-spectral clustering, sparsest cut or minimizing conductance \citep{Buhler09,Hein10,Szlam10,Arora09}.

\section{Analysis}\label{sec:thm}
%%%%%%%%%%%%%%%%%%%%%%%%%%%%%%%%%%%%%%%%%%
%
Our asymptotic analysis shows how RMD helps control of cut-ratio $q$ introduced in Sec.~2.
Assume the data set $\{x_1,\ldots,x_n\}$  is drawn i.i.d. from an underlying density $f$ in $\mathbb{R}^d$. Let $G=(V,E)$ be the unweighted RMD graph. Given a separating hyperplane $S$, denote $C^+$,$C^-$ as two subsets of $C$ split by $S$, $\eta_d$ the volume of unit ball in $\mathbb{R}^d$. Assume the density $f$ satisfies:

\noindent
\textbf{Regularity conditions:} $f(\cdot)$ has a compact support, and is continuous and bounded: $f_{max} \geq f(x) \geq f_{min}>0$. It is smooth, i.e. $||\nabla f(x)||\leq\lambda$, where $\nabla f(x)$ is the gradient of $f(\cdot)$ at $x$. There is no flat regions, i.e. $\forall \sigma>0$, $\mathcal{P}\left\{y: |f(y)-f(x)|<\sigma\right\}\leq M\sigma$ for all $x$ in the support, where $M$ is a constant.

First we show the asymptotic consistency of the rank $R(y)$ at some point $y$. The limit of $R(y)$ is $p(y)$, which is the complement of the volume of the level set containing $y$. Note that $p$ exactly follows the shape of $f$, and always ranges in $[0,1]$ no matter how $f$ scales.
\begin{theorem}\label{rank-pvalue}
Assume $f(x)$ satisfies the above regularity conditions. As $n\rightarrow\infty$, we have
\begin{equation}
    R(y)\rightarrow p(y):= \int_{\left\{x:f(x)\leq
f(y)\right\}}f(x)dx.
\end{equation}
\end{theorem}

The proof involves the following two steps:
\begin{itemize}
  \item[1.] The expectation of the empirical rank $\mathbb{E}\left[R(y)\right]$ is shown to converge to $p(y)$ as $n\rightarrow\infty$.
  \item[2.] The empirical rank $R(y)$ is shown to concentrate at its expectation as $n\rightarrow\infty$.
\end{itemize}
Details can be found in the supplementary.
Small/Large $R(x)$ values correspond to low/high density respectively. $R(x)$ asymptotically converges to an integral expression, so it is smooth (Fig.\ref{f.pdf_rank}). Also $p(x)$ is uniformly distributed in $[0,1]$. This makes it appropriate to modulate the degrees with control of minimum, maximal and average degree.

Next we study RCut(NCut) induced on unweighted RMD graph. Assume for simplicity that each node $v$ is connected to exactly $k_\lambda(v)$ nearest neighbors of Eq.(\ref{eq:degree}). The limit cut expression on RMD graph involves an additional adjustable term which varies point-wise according to the density.
% For technical simplicity, we assume RMD graph ideally connects each point $x$ to its $k_\lambda$ closest neighbors.

\begin{theorem}\label{part2}
Assume $f$ satisfies the above regularity conditions and also the general assumptions in \citet{Maier1}. $S$ is a fixed hyperplane in $\mathbb{R}^d$. For unweighted RMD graph, set the degrees of points according to Eq.(\ref{eq:degree}), where $\lambda\in(0,1)$ is a constant. Let $\rho(x)=\lambda+2(1-\lambda)p(x)$. Assume $k_n/n\rightarrow{0}$. In case $d$=1, assume $k_n/\sqrt{n}\rightarrow\infty$; in case $d\geq$2 assume $k_n/\log{n}\rightarrow\infty$. Then as $n\rightarrow\infty$ we have that:
\begin{equation} \label{eq:rcut}
    \frac{1}{k_n}\sqrt[d]{\frac{n}{k_n}}RCut_n(S)\longrightarrow  C_d B_S \int_S{f^{1-\frac{1}{d}}(s)\rho^{1+\frac{1}{d}}(s)ds}.
\end{equation}
\begin{equation} \label{eq:ncut}
    \sqrt[d]{\frac{n}{k_n}}NCut_n(S)\longrightarrow  C_d B_S \int_S{f^{1-\frac{1}{d}}(s)\rho^{1+\frac{1}{d}}(s)ds}.
\end{equation}
where $C_d = \frac{2\eta_{d-1}}{(d+1)\eta_d^{1+1/d}}$, $B_S=\left(\mu(C^+)^{-1}+\mu(C^-)^{-1}\right)$,  and $\mu(C^{\pm})=\int_{C^{\pm}}f(x)dx$.
\end{theorem}
The proof shows the convergence of the cut term and balancing term respectively:
\begin{eqnarray}
    &\frac{1}{nk_n}\sqrt[d]{\frac{n}{k_n}}cut_n(S)
    \rightarrow C_d\int_S{f^{1-\frac{1}{d}}(s)\rho^{1+\frac{1}{d}}(s)ds}, \\
    &n\frac{1}{|V^\pm|}\rightarrow
    \frac{1}{\mu(C^\pm)}, \,\,\,\,  nk_n\frac{1}{vol(V^\pm)}\rightarrow
    \frac{1}{\mu(C^\pm)}. \label{eq:term2N}
    %n\frac{1}{|V^\pm|} &\rightarrow \frac{1}{2\mu(C^{\pm})} \label{eq:term3}
\end{eqnarray}
It is an extension of \citet{Maier1}. Details are in the supplementary.

%The proofs can be found in supplementary section.\\

\noindent
{\bf Imbalanced Data \& RMD Graphs:} \\
In the limit cut behavior, without our $\rho$ term, the balancing term $B_S=1/\alpha(1-\alpha)$ could induce a larger RCut(NCut) value for density valley cut than balanced cut when the underlying data is imbalanced, i.e. $\alpha$ is small.
Applying our parameterization scheme appends an additional term $\rho(s)=(\lambda+2(1-\lambda)p(s))$ in the limit-cut expressions.
$\rho(s)$ is monotonic in the p-value and so the cut-value at low/high density regions can be further reduced/increased. Indeed for small $\lambda$ value, cuts $S$ near peak densities have $p(s)\approx 1$ and so $\rho(s) \approx (2)^{1+\frac{1}{d}}$, while near valleys we have $\rho(s)\approx (\lambda)^{1+\frac{1}{d}}\ll 1$. This has a direct bearing on cut-ratio, $q$ since small $\lambda$ can reduce the cut-ratio $q$ for a given $y$ (see Fig.1) and leads to better control on imbalanced data. In summary, this analysis shows that RMD graphs used in conjunction with optimization framework of Fig.~\ref{f.framework} can adapt to varying levels of imbalanced data.

%%%%%%%%%%%%%%%%%%%%%%%%%%%%%%%%%%%%%%%%%%
\section{Experiments}\label{sec:experiment}
%%%%%%%%%%%%%%%%%%%%%%%%%%%%%%%%%%%%%%%%%%
Experiments in this section involve both synthetic and real data sets. We focus on imbalanced data by randomly sampling from different classes disproportionately.  For comparison purposes we compare RMD graph with full-RBF, $\epsilon$-graph, RBF $k$-NN, $b$-matching graph \citep{JebWanCha09} and full graph with adaptive RBF (full-aRBF) \citep{Zelnik04}. We view each as a parametric family of graphs parameterized by their relevant parameters and optimize over different parameters as described in Sec.~\ref{sec:RMD_idea} and Eq.(\ref{eq:selection}). For RMD graphs we also optimize over $\lambda$ in addition. Error rates are averaged over 20 trials.

For clustering experiments we apply both RCut and NCut, but focus mainly on NCut for brevity (NCut is generally known to perform better). We report performance by evaluating how well the cluster structures match the ground truth labels, as is the standard criterion for partitional clustering \citep{xu05}. For instance consider Table 1 where error rates for USPS symbols 1,8,3,9 are tabulated. We follow our procedure outlined in Sec.~\ref{sec:RMD_idea} and find the optimal partition that minimizes Eq.(\ref{eq:selection}) \emph{agnostic} to the correspondence between samples and symbols. Errors are then reported by looking at mis-associations. %To validate our framework of Fig.~\ref{f.framework} we also

For SSL experiments we randomly pick labeled points among imbalanced sampled data, guaranteeing at least one labeled point from each class. SSL algorithms such as RCut-based GRF and NCut-based GTAM are applied on parameterized graphs built from partially labeled data, and generate various partitions. Again we follow our procedure outlined in Sec.~\ref{sec:RMD_idea} and find the optimal partition that minimizes Eq.(\ref{eq:selection}) agnostic to ground truth labels. Then labels for unlabeled data are predicted based on the selected partition and compared against the \emph{unknown} true labels to produce the error rates.

\noindent \underline{\it Time Complexity:}
%The time complexity of U-statistic rank computation is $O(Bdn^2logn)$, where $B$ is a small constant, 5 in our experiments.
RMD graph construction is $O(dn^2logn)$ (similar to $k$-NN graph). Computing cut value and checking cluster size for a partition takes $O(n^2)$. So if totally $D$ graphs are parameterized; complexity of learning algorithm is $T$, the time complexity is $O(D(dn^2logn+T))$.

\noindent \underline{\it Tuning Parameters:} Note that parameters including $\lambda,k,\sigma$ that characterize the graphs are variables to be optimized in Eq.(\ref{eq:selection}). The only parameters left are: \\
(a) $k_0$ in the baseline graph. This is fixed to be $\sqrt{n}$. \\
(b) Imbalanced size threshold $\delta$. We fix this a priori to be about $0.05$, i.e., 5\% of all samples.

\noindent \underline{\it Evaluation against Oracle:} To evaluate the effectiveness of our framework (Fig.~\ref{f.framework}) and RMD parameterization, we compare against an ORACLE that is tuned to both ground truth labels as well as imbalanced proportions.

%Some general simulation parameters are:\\
%{\bf (1)} We employ U-statistic technique in rank computation to reduce variance (Sec.\ref{subsec:rank}), with $B=5$.\\
%{\bf (2)} All error rate results are averaged over 20 trials.\\
%Other parameters will be specified below.

\subsection{Synthetic Illustrative Example}
%\noindent
%{\bf Synthetic Datasets:}
%%%%%%%%%%%%%%%%%%%%%%%%%%%%%%%%%%%%%%%%%%
%\subsection{Synthetic DataSets}\label{subsec:syn}
%%%%%%%%%%%%%%%%%%%%%%%%%%%%%%%%%%%%%%%%%%
\begin{figure*}[htb]
\begin{centering}
\begin{minipage}[t]{.24\textwidth}
\includegraphics[width = 1\textwidth]{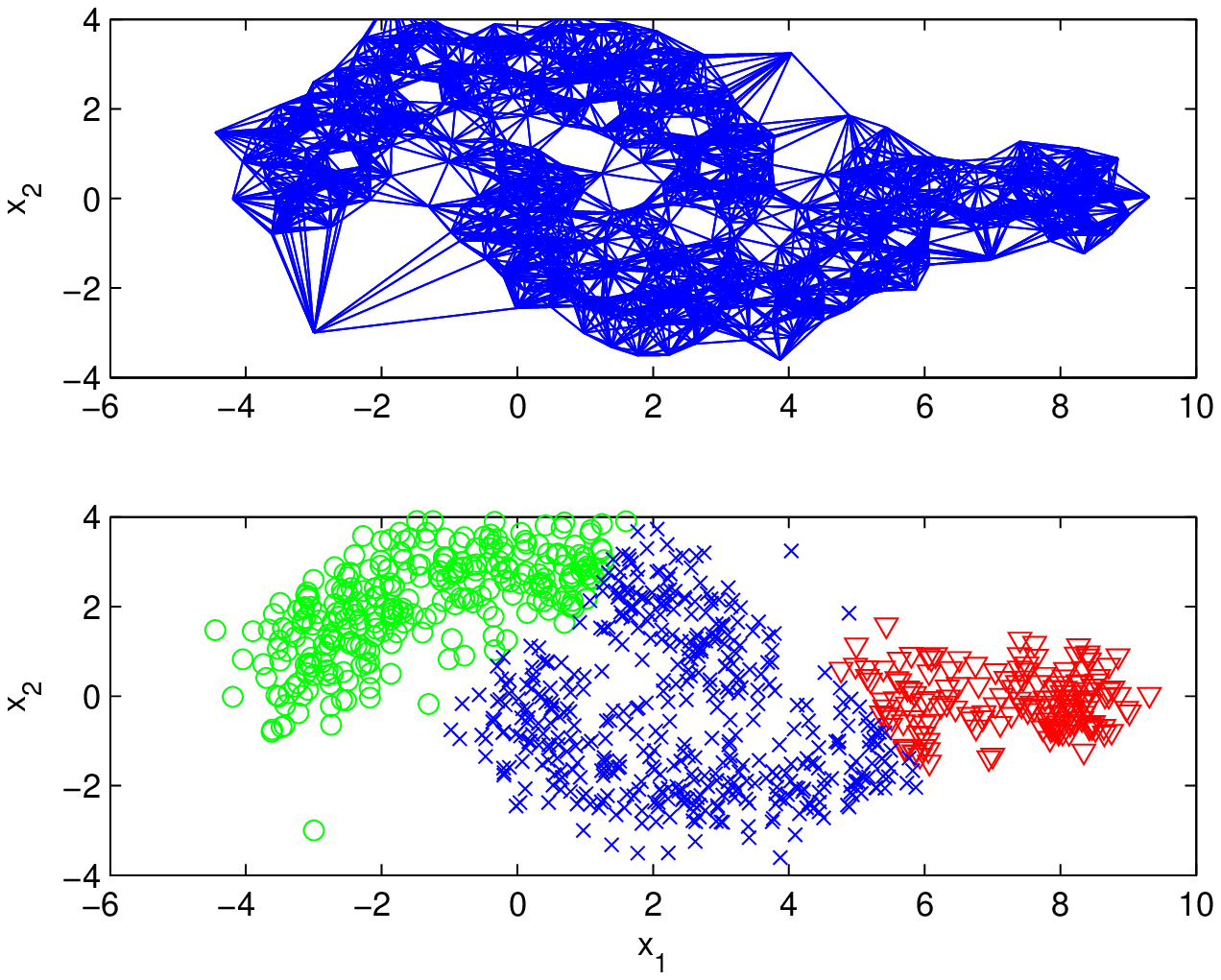}
\makebox[3cm]{(a) $k$-NN}
\end{minipage}
\begin{minipage}[t]{.24\textwidth}
\includegraphics[width = 1\textwidth]{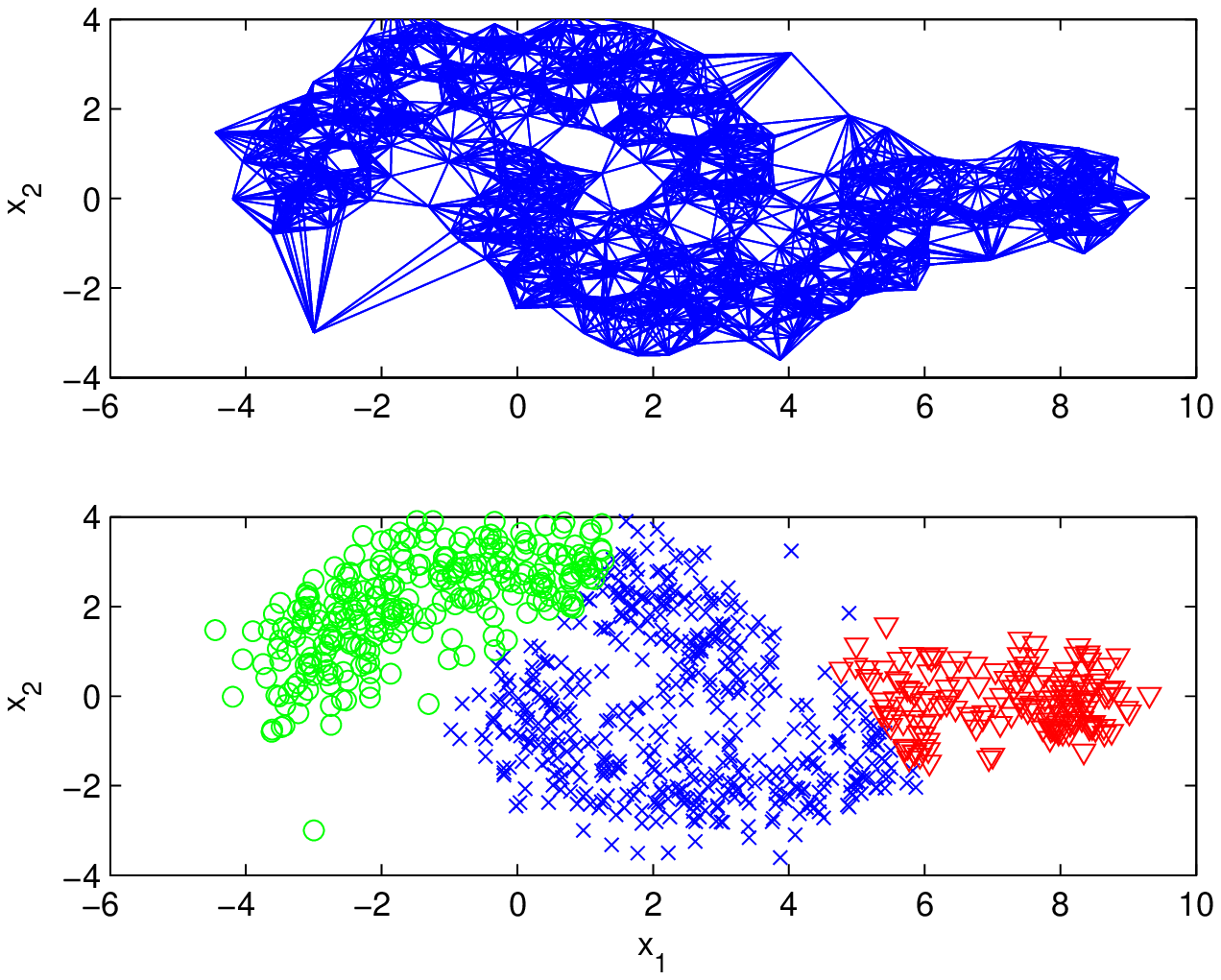}
\makebox[3cm]{(b) $b$-matching}
\end{minipage}
\begin{minipage}[t]{.24\textwidth}
\includegraphics[width = 1\textwidth]{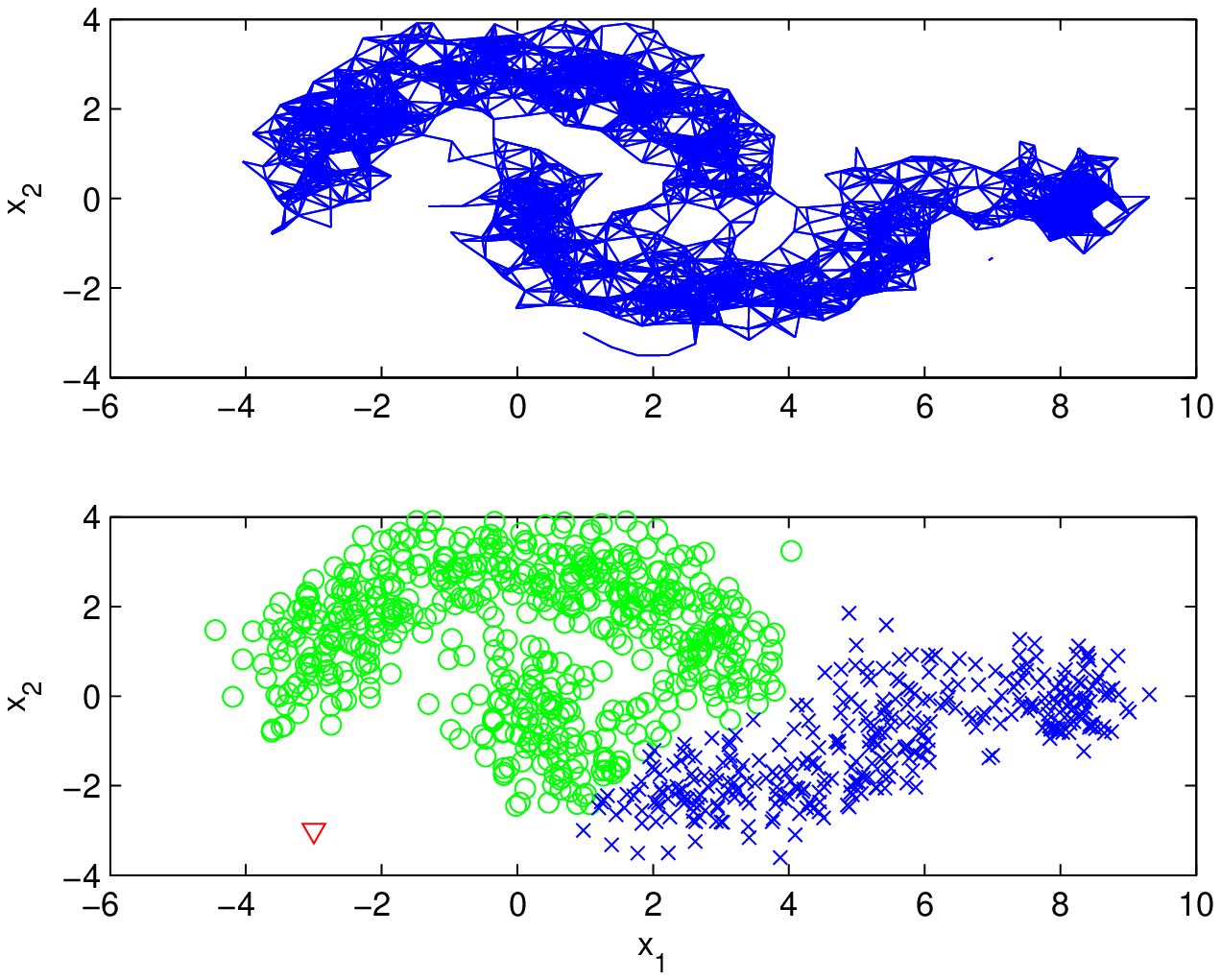}
\makebox[3cm]{(c) $\epsilon$-graph(full-RBF)}
\end{minipage}
\begin{minipage}[t]{.24\textwidth}
\includegraphics[width = 1\textwidth]{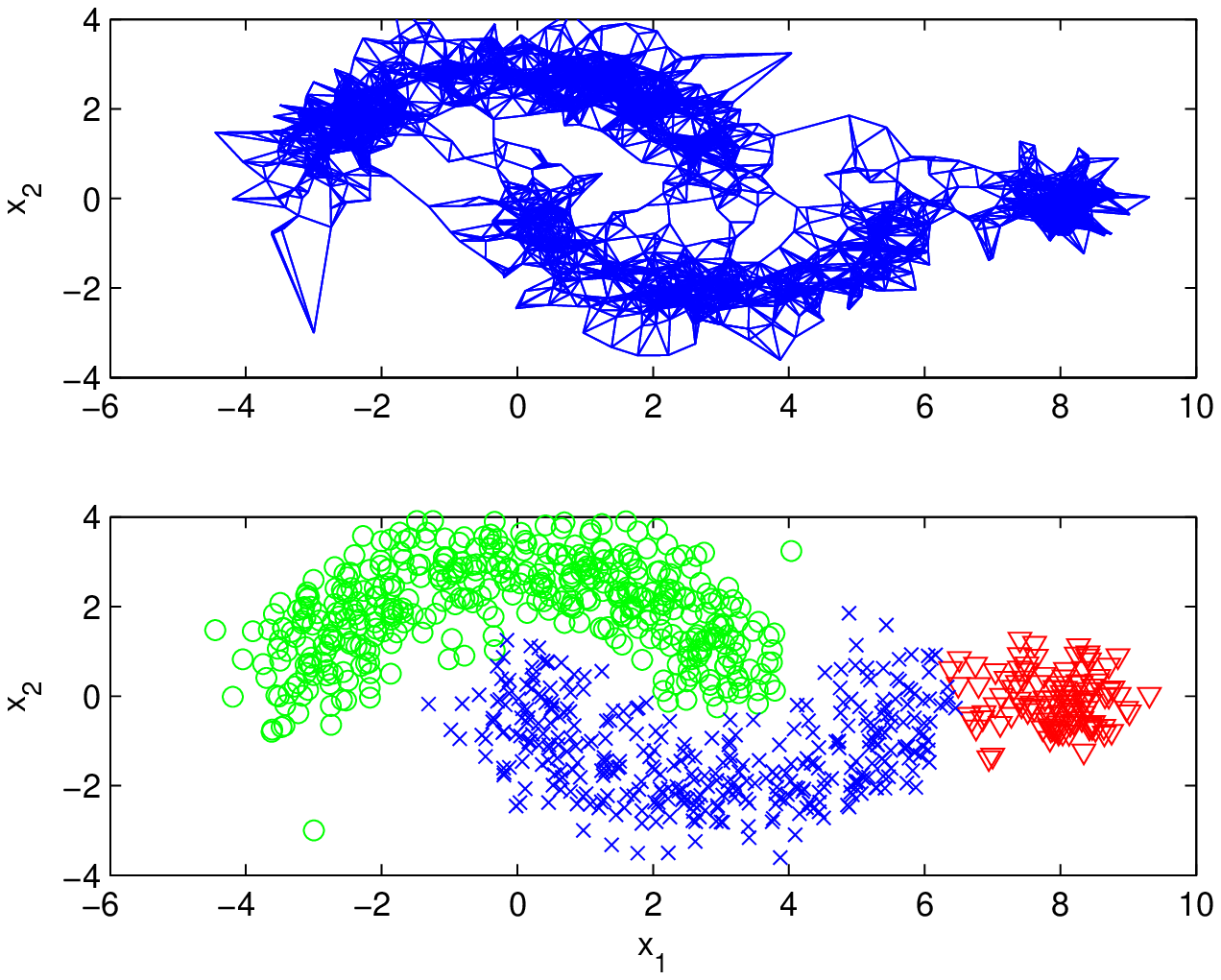}
\makebox[3cm]{(d) RMD}
\end{minipage}
\caption{Clustering results of 3-partition SC on 2 moons and 1 gaussian data set. SC on full-RBF($\epsilon$-graph) completely fails due to the outlier. For $k$-NN and $b$-matching graphs SC cannot recognize the long winding low-density regions between 2 moons, and fails to find the rightmost small cluster. Our method sparsifies the graph at low-density regions, allowing to cut along the valley, detect the small cluster and is robust to outliers.}
\label{fig:complex_shape}
\end{centering}
\end{figure*}

Consider a multi-cluster complex-shaped data set, which is composed of 1 small Gaussian and 2 moon-shaped proximal clusters shown in Fig.\ref{fig:complex_shape}. Sample size $n=1000$ with the rightmost small cluster $10\%$ and two moons $45\%$ each. This example is only for illustrative purpose with a single run, so we did not parameterize the graph or apply step (4). We fix $\lambda = 0.5$, and choose $k=l=30$, $\epsilon=\sigma=\tilde{d}_k$, where $\tilde{d}_k$ is the average $k$-NN distance. Model-based approaches can fail on such dataset due to the complex shapes of clusters. The 3-partition SC based on RCut is applied. On $k$-NN and $b$-matching graphs SC fails for two reasons: (1) SC cuts at balanced positions and cannot detect the rightmost small cluster; (2) SC cannot recognize the long winding low-density regions between 2 moons because there are too many spurious edges and the Cut value along the curve is big. SC fails on $\epsilon$-graph(similar on full-RBF) because the outlier point forms a singleton cluster, and also cannot recognize the low-density curve. Our RMD graph significantly sparsifies the graph at low-densities, enabling SC to cut along the valley, detect small clusters and reject outliers.

\subsection{Real Experiments}
%\noindent
%{\bf Real Datasets:}

We focus on imbalanced settings for several real datasets. We construct $k$-NN, $b$-match, full-RBF and RMD graphs all combined with RBF weights, but do not include the $\epsilon$-graph because of its overall poor performance \citep{JebWanCha09}.
Our sample size varies from 750 to 1500.
We discretize not only $\lambda$ but also $k$, $\sigma$ to parameterize graphs.
%For example, the result of RBF $k$-NN graph is chosen based on optimizing the following expression:
%\begin{eqnarray}
%% \nonumber to remove numbering (before each equation)
%  & J(\delta)=\min_{k,\sigma}\{Cut\left(C(k,\sigma),\bar{C}(k,\sigma)\right)\} \\
%\nonumber
%  & s.t. ~~\min\{|C(k,\sigma)|,|\bar{C}(k,\sigma)|\}\geq \delta n
%\end{eqnarray}
%where, $C(k,\sigma),\,\bar C(k,\sigma)$ denotes the RCut partition obtained on the RBF $k$-NN graph with nearest neighbor parameter $k$ and RBF parameter $\sigma$.
% The optimization problem is non-convex but involves search over a small number of parameters.
% We discretized the parameters in our experiments.
We vary $k$ in $\{5,10,20,30,\ldots,100,120,150\}$.
While small $k$ may lead to disconnected graphs this is not an issue for us since singleton cluster candidates are ruled infeasible in PCut. Also notice that for $\lambda=1$, RMD graph is identical to $k$-NN graph. For RBF parameter $\sigma$ it has been suggested to be of the same scale as the average $k$-NN distance $\tilde{d}_k$ \citep{WanJebCha08}. This suggests a discretization of $\sigma$ as $2^j \tilde{d}_k$ with $j=-3,\,-2,\ldots,\,3$. We discretize $\lambda \in [0,\,1]$ and varied in steps of $0.2$.

In the model selection step Eq.(\ref{eq:selection}), cut values of various partitions are evaluated on a same $k_0$-NN graph with $k_0=30, \sigma = \tilde{d}_{30}$ before selecting the min-cut partition. The true number of clusters/classes $K$ is supposed to be known. We assume meaningful clusters are at least $5\%$ of the total number of points, $\delta=0.05$. We set the GTAM parameter $\mu=0.05$ as in \citep{JebWanCha09} for the SSL tasks, and each time 20 randomly labeled samples are chosen with at least one sample from each class.

%\begin{figure}[tb]
%\begin{centering}
%\begin{minipage}[t]{.23\textwidth}
%\includegraphics[width = 1\textwidth]{USPS8v9_SC.eps}
%\makebox[3cm]{\small (a) SC}
%\end{minipage}
%\begin{minipage}[t]{.23\textwidth}
%\includegraphics[width = 1\textwidth]{USPS8v9_GTAM.eps}
%\makebox[3cm]{\small (b) GTAM}
%\end{minipage}
%%\begin{minipage}[t]{.34\textwidth}
%%\includegraphics[width = 1\textwidth]{USPS8v9_GTAM.eps}
%%\makebox[3.7cm]{\small (c) GTAM on USPS 8vs9}
%%\end{minipage}
%\caption{\small Error rate performance of SC and GTAM on USPS 8vs9 with varying levels of unbalancedness. We omitted GRF since the results are qualitatively similar. Our method adapts to different levels of unbalancedness much better than traditional graphs. Furthermore, when data is very unbalanced (big $n_9/n_8$), varying $k,\sigma$ does not really help; decreasing $\lambda$ adapts the algorithm well.}
%\label{fig:USPS8v9}
%\end{centering}
%\vspace{-0.2in}
%\end{figure}

%\begin{wrapfigure}{r}{.45\textwidth}
%\centering
%%\begin{minipage}[t]{.23\textwidth}
%\includegraphics[width = .45\textwidth]{USPS8v9_SC.eps}
%%\makebox[3cm]{\small (a) SC}
%%\end{minipage}
%%\begin{minipage}[t]{.23\textwidth}
%%\includegraphics[width = .45\textwidth]{USPS8v9_GTAM.eps}
%%\makebox[3cm]{\small (b) GTAM}
%%\end{minipage}
%%\vspace{-15pt}
%\caption{\small Error rates of SC on USPS 8vs9 with varying levels of imbalance. Our method adapts to different levels of imbalance much better than traditional graphs.}
%\label{fig:USPS8v9}
%% \vspace*{-10pt}
%\end{wrapfigure}

\begin{figure*}[tb]
\begin{centering}
\begin{minipage}[t]{.32\textwidth}
\includegraphics[width = 1\textwidth]{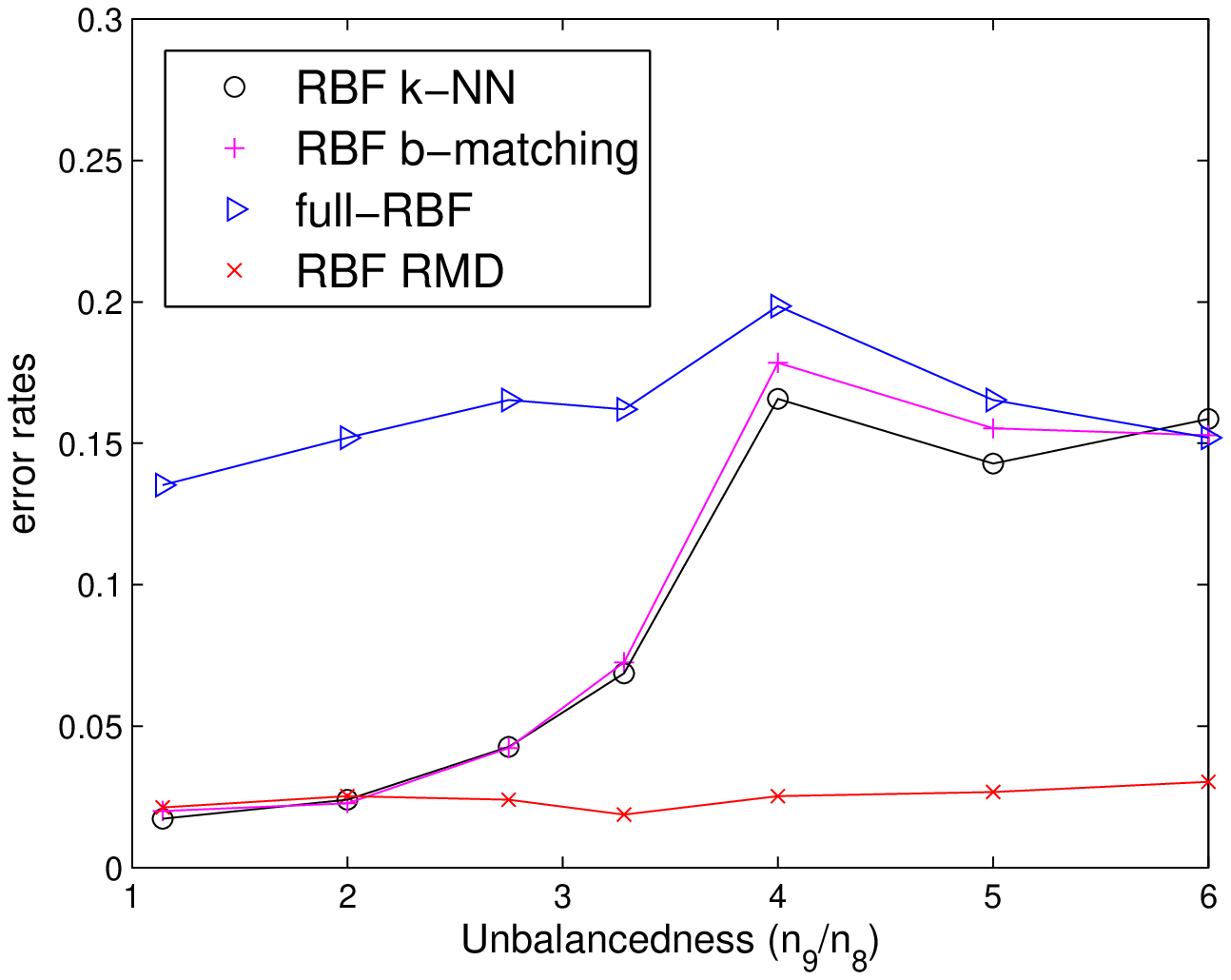}
\makebox[4cm]{(a) SC(clustering)}
\end{minipage}
\begin{minipage}[t]{.32\textwidth}
\includegraphics[width = 1\textwidth]{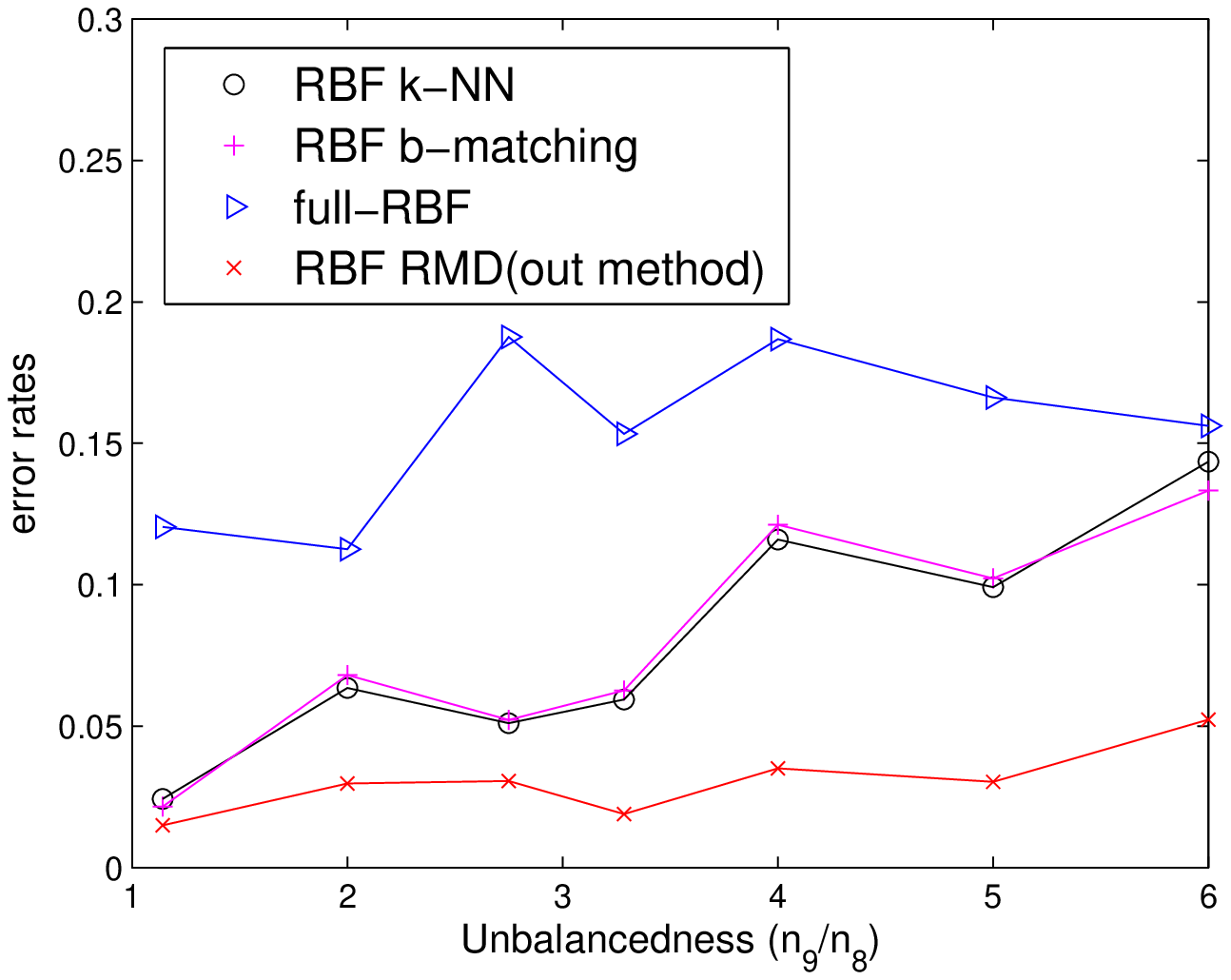}
\makebox[4cm]{(b) GRF(SSL)}
\end{minipage}
\begin{minipage}[t]{.32\textwidth}
\includegraphics[width = 1\textwidth]{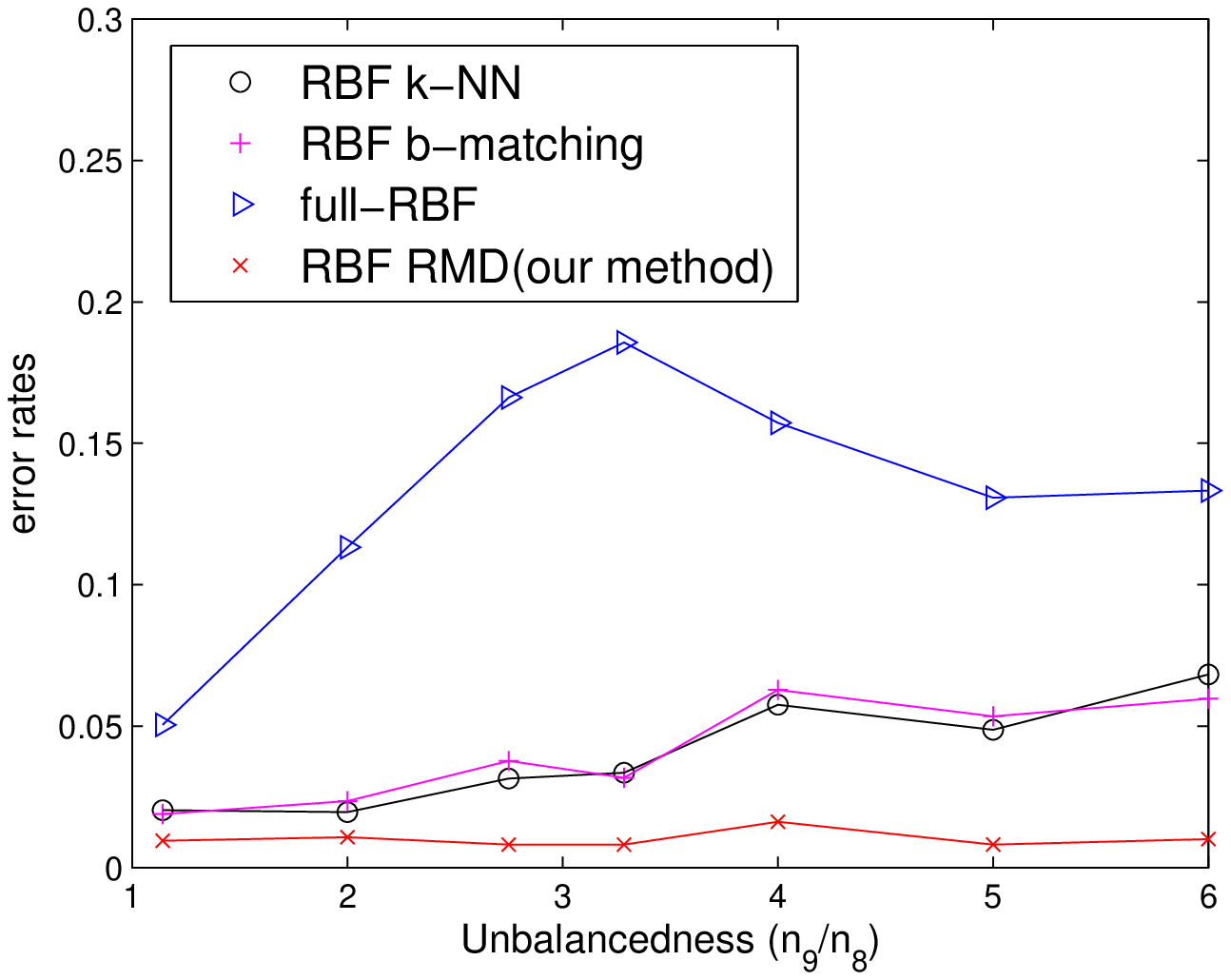}
\makebox[4cm]{(c) GTAM(SSL) }
\end{minipage}
\caption{Error rates of SC and SSL algorithms on USPS 8vs9 with varying levels of imbalancedness. Our RMD scheme remains competitive when the data is balanced, and adapts to imbalancedness much better than traditional graphs.}
\label{fig:USPS8v9}
\end{centering}
%\vspace{-10pt}
\end{figure*}

\noindent
\textbf{Varying Imbalancedness:} \\
Here we use 8 vs 9 in the 256-dim USPS digit data set and randomly sample 750 points with different levels of imbalancedness. Normalized SC, GRF and GTAM are then applied. Fig.\ref{fig:USPS8v9} shows that when the underlying clusters/classes are balanced, our RMD method performs as well as traditional graphs; as the imbalancedness increases, the performance of other graphs degrades, while our method can adapt to different levels of imbalancedness.
\begin{table*}[tb]
\caption{ Imbalancedness of data sets. }
\begin{center}
\begin{tabular}{|c||c|}
  \hline
  Data sets     & \#samples per cluster  \\
  \hline\hline
  2-cluster(eg. USPS 8vs9 etc.)     & 150/600 \\    \hline
  3-cluster(eg. SagImg 3/4/5 etc.)  & 200/400/600    \\ \hline
  4-cluster(eg. USPS 1/8/3/9 etc.)  & 200/300/400/500     \\
  \hline
\end{tabular}
\end{center}
\label{tab:Imbalanced_demo}
\end{table*}

\begin{table*}[tb]
\caption{ Error rates of normalized SC on various graphs for imbalanced real data sets. Our method performs significantly better than other methods. First row (``BO'' Balanced Oracle) shows RBF $k$-NN results on imbalanced data with $k,\sigma$ tuned using ground truth labels but on balanced data. Last row (``O'' Oracle) shows the best ORACLE results of RBF RMD on imbalanced data. }
\vspace{-10pt}
\begin{center}
\begin{tabular}{|c||c|c|c|c|c|c|c|c|c|c|}
  \hline
  \multirow{2}{*}{Error Rates(\%)}   &   \multicolumn{2}{c|}{USPS}  &   \multicolumn{3}{c|}{SatImg}  &   \multicolumn{3}{c|}{OptDigit}   & \multicolumn{2}{c|}{LetterRec} \\
  \cline{2-11}
  & 8vs9 & 1,8,3,9 & 4vs3 & 3,4,5 & 1,4,7 & 9vs8 & 6vs8 & 1,4,8,9 & 6vs7 & 6,7,8 \\
  \hline\hline
  RBF $k$-NN(BO)    & 33.20 & 17.60 & 15.76 & 22.08 & 25.28 & 15.17 & 11.15  & 30.02 & 7.85 & 38.70     \\
  RBF $k$-NN        & 16.67 & 13.21 & 12.80 & 18.94 & 25.33 & 9.67  & 10.76  & 26.76 & 4.89 & 37.72 \\
  RBF $b$-match  & 17.33 & 12.75 & 12.73 & 18.86 & 25.67 & 10.11  & 11.44  & 28.53 & 5.13 & 38.33 \\
  full-RBF          & 19.87 & 16.56 & 18.59 & 21.33 & 34.69 & 11.61 & 15.47 & 36.22 & 7.45 & 35.98 \\
  full-aRBF         & 18.35 & 16.26 & 16.79 & 20.15 & 35.91 & 10.88 & 13.27 & 33.86 & 7.58 & 35.27 \\
  RBF RMD           & 4.80  & 9.66 & 9.25 & 16.26 & 20.52 & 6.35  & 6.93  & 23.35 & 3.60 & 28.68 \\
  RBF RMD(O)        & 3.13  & 7.89 & 8.30 & 14.19 & 18.72 & 5.43  & 6.27  & 19.71 & 3.02 & 25.33 \\
  \hline
\end{tabular}
\end{center}
\label{tab:real_SC}
\end{table*}

\noindent
\textbf{Other Real Data Sets:} \\
We apply SC and SSL algorithms on several other real data sets including USPS(256-dim), Statlog landsat satellite images(4-dim), letter recognition images(16-dim) and optical recognition of handwritten digits(16-dim) \citep{uci10}.
%We randomly sample 150/600, 200/400/600, 200/300/400/500 points for 2,3,4-class cases, with corresponding orders of class indices listed in Table \ref{tab:real_SC},\ref{tab:real_SSL}.
We sample data sets in an imbalanced way shown in Table \ref{tab:Imbalanced_demo}.

In Table \ref{tab:real_SC} the first row is the imbalanced results of RBF $k$-NN using ORACLE $k,\sigma$ parameters tuned with ground-truth labels on balanced data for each data set (300/300, 250/250/250, 250/250/250/250 samples for 2,3,4-class cases). Comparison of first two rows reveals that the ORACLE choice on balanced data may not be suitable for imbalanced data, while our PCut framework, although agnostic, picks more suitable $k,\sigma$ for RBF $k$-NN.
The last row presents ORACLE results on RBF RMD tuned to imbalanced data. This shows that our PCut on RMD, agnostic of true labels, closely approximates the oracle performance. Also, both tables show that
our RMD graph parameterization performs consistently better than other methods. %graph parameterization performs consistently better than other methods.
%The difference in performance between the agnostic $k$-nn but optimized using our framework of Fig.~\ref{f.framework} vs the Oracle $k$-nn demonstrates the utility of our method. The last column also demonstrates the value of our RMD graph construction over other methods.

%Tab.\ref{tab:real_SC},\ref{tab:real_SSL} demonstrates two important points:
%
%(1) Compare the first 2 rows of Tab.\ref{tab:real_SC}} and we see that even the best {\bf ORACLE} parameters obtained under balanced settings can be unsuitable to the unbalanced settings for the same task on the same data set. Our optimization framework AGNOSTIC of ground truth labels can improve performance.
%
%(2) Tab.\ref{tab:real_SC},\ref{tab:real_SSL} also show that our method (last row) performs consistently better, which demonstrates that our RMD scheme can further boost the performance for unbalanced data.

\begin{table*}[h]
\caption{Error rate performance of GRF and GTAM for imbalanced real data sets. Our method performs significantly better than other methods.}
\vspace{-10pt}
\begin{center}
\begin{tabular}{|c|c||c|c|c|c|c|c|c|c|c|}
  \hline
%  \multirow{2}{*}{\multicolumn{2}{c}{Error Rates(\%)}}
  \multicolumn{2}{|c||}{\multirow{2}{*}{Error Rates(\%)}}  &   \multicolumn{2}{c|}{USPS}  &   \multicolumn{2}{c|}{SatImg}  &   \multicolumn{3}{c|}{OptDigit}   & \multicolumn{2}{c|}{LetterRec} \\
  \cline{3-11}
  \multicolumn{2}{|c||}{}  & 8vs6 & 1,8,3,9 & 4vs3 & 1,4,7 & 6vs8 & 8vs9 & 6,1,8 & 6vs7 & 6,7,8 \\
  \hline\hline
  \multirow{4}{*}{GRF}
    & RBF $k$-NN            & 5.70 & 13.29 & 14.64 & 16.68 & 5.68  & 7.57  & 7.53 & 7.67 & 28.33 \\
    & RBF $b$-matching      & 6.02 & 13.06 & 13.89 & 16.22 & 5.95  & 7.85  & 7.92 & 7.82 & 29.21 \\
    & full-RBF              & 15.41 & 12.37 & 14.22 & 17.58 & 5.62 & 9.28 & 7.74 & 11.52 & 28.91 \\
    & full-aRBF             & 12.89 & 11.74 & 13.58 & 17.86 & 5.78 & 8.66 & 7.88 & 10.10 & 28.36 \\
    & RBF RMD               & 1.08  & 10.24 & 9.74 & 15.04 & 2.07  & 2.30  & 5.82 & 5.23 & 27.24 \\
  \hline
  \multirow{4}{*}{GTAM}
    & RBF $k$-NN            & 4.11  & 10.88 & 26.63 & 20.68 & 11.76 & 5.74  & 12.68 & 19.45 & 27.66 \\
    & RBF $b$-matching      & 3.96  & 10.83 & 27.03 & 20.83 & 12.48 & 5.65  & 12.28 & 18.85 & 28.01 \\
    & full-RBF              & 16.98  & 11.28 & 18.82 & 21.16 & 13.59 & 7.73 & 13.09 & 18.66 & 30.28 \\
    & full-aRBF             & 13.66  & 10.05 & 17.63 & 22.69 & 12.15 & 7.44 & 13.09 & 17.85 & 31.71 \\
    & RBF RMD               & 1.22  & 9.13 & 18.68 & 19.24 & 5.81  & 3.12  & 10.73 & 15.67 & 25.19 \\
  \hline
\end{tabular}
\end{center}
\label{tab:real_SSL}
\end{table*}

\subsection{Small Cluster Detection: PCut with varying Partition-size Threshold}
We illustrate how our method can be used to find small-size clusters. This type of problem may arise in community detection in large real networks, where graph-based approaches are popular but small-size community detection is difficult \citep{Shah10}. The dataset depicted in Fig.\ref{fig:multiple_cuts} has 1 large and 2 small proximal Gaussian components along $x_1$ axis: $\sum^{3}_{i=1}\alpha_iN(\mu_i,\Sigma_i)$, where $\alpha_1:\alpha_2:\alpha_3=2:8:1$, $\mu_1$=[-0.7;0], $\mu_2$=[4.5;0], $\mu_3$=[9.7;0], $\Sigma_1=I, \Sigma_2=diag(2,1), \Sigma_3=0.7I$. Binary weight is adopted.
Fig.\ref{fig:multiple_cuts}(a) shows a plot of cut values for a baseline $k$-NN graph for different cut positions averaged over 20 Monte Carlo runs. The cut-value plot resembles the underlying density. The two density valleys are at imbalanced positions with rightmost cluster smaller but the leftmost valley deeper. %cluster, but has a deeper valley.
%
%\begin{wrapfigure}{r}{.5\textwidth}
%\vspace{-20pt}
%\centering
%\includegraphics[width=.47\textwidth]{3g_multicuts_2.eps}
%\vspace{-10pt}
%\caption{\small Gaussian mixture with 3 unbalanced Gaussian components. Depicted is the result of a single realization. Our method is able to discover two small clusters. The left larger one is detected for a larger $\delta$, where the right smaller one is viewed as outliers. When even reducing $\delta$, the right smaller one is detected(see Eq.(\ref{eq:selection})).}
%\label{fig:multicluster}
%\end{wrapfigure}

\begin{figure*}[htbp]
\begin{centering}
\begin{minipage}[t]{.32\textwidth}
\includegraphics[width = 1\textwidth]{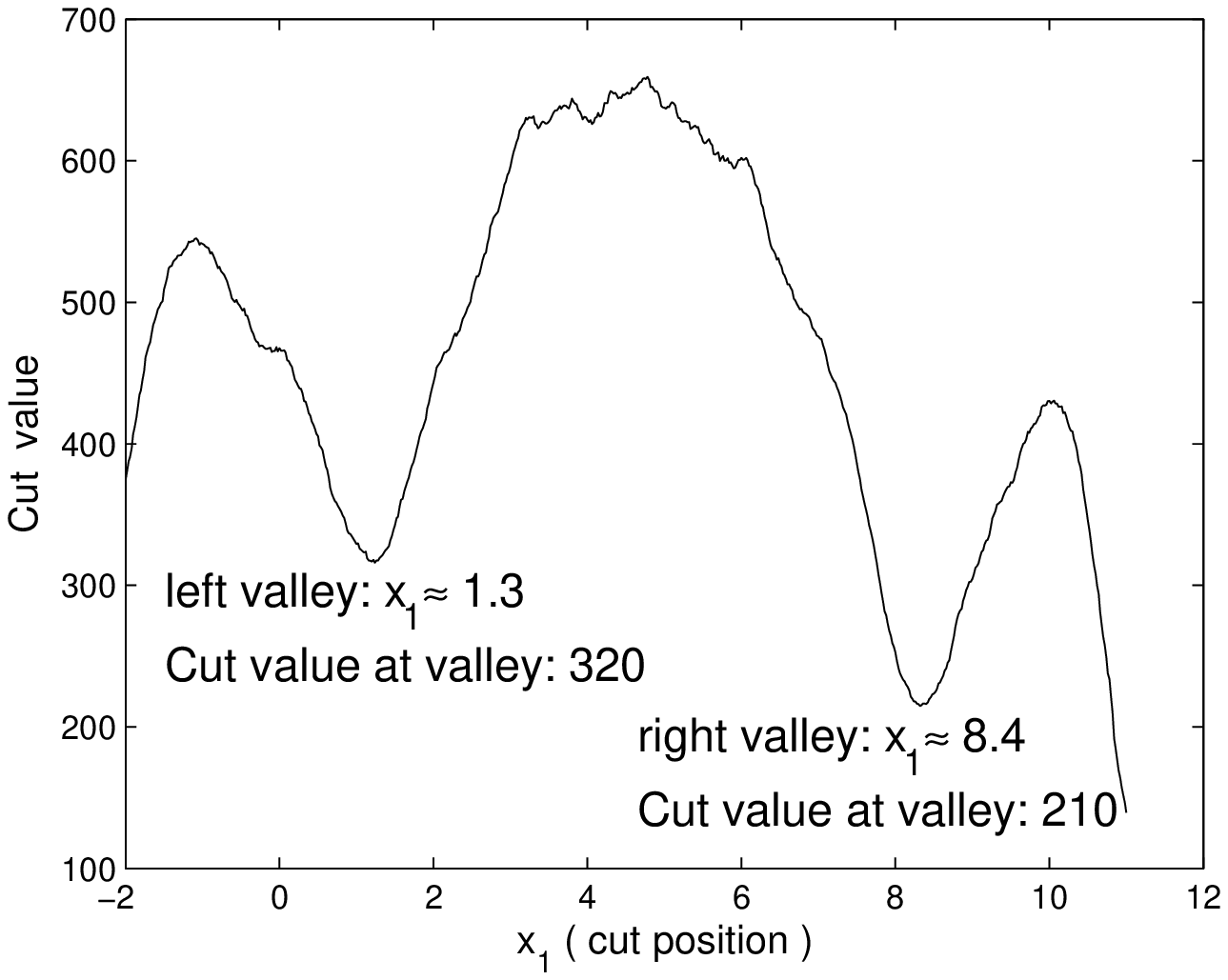}
\makebox[4.5cm]{(a) Cut value vs. cut position}
\end{minipage}
\begin{minipage}[t]{.32\textwidth}
\includegraphics[width = 1\textwidth]{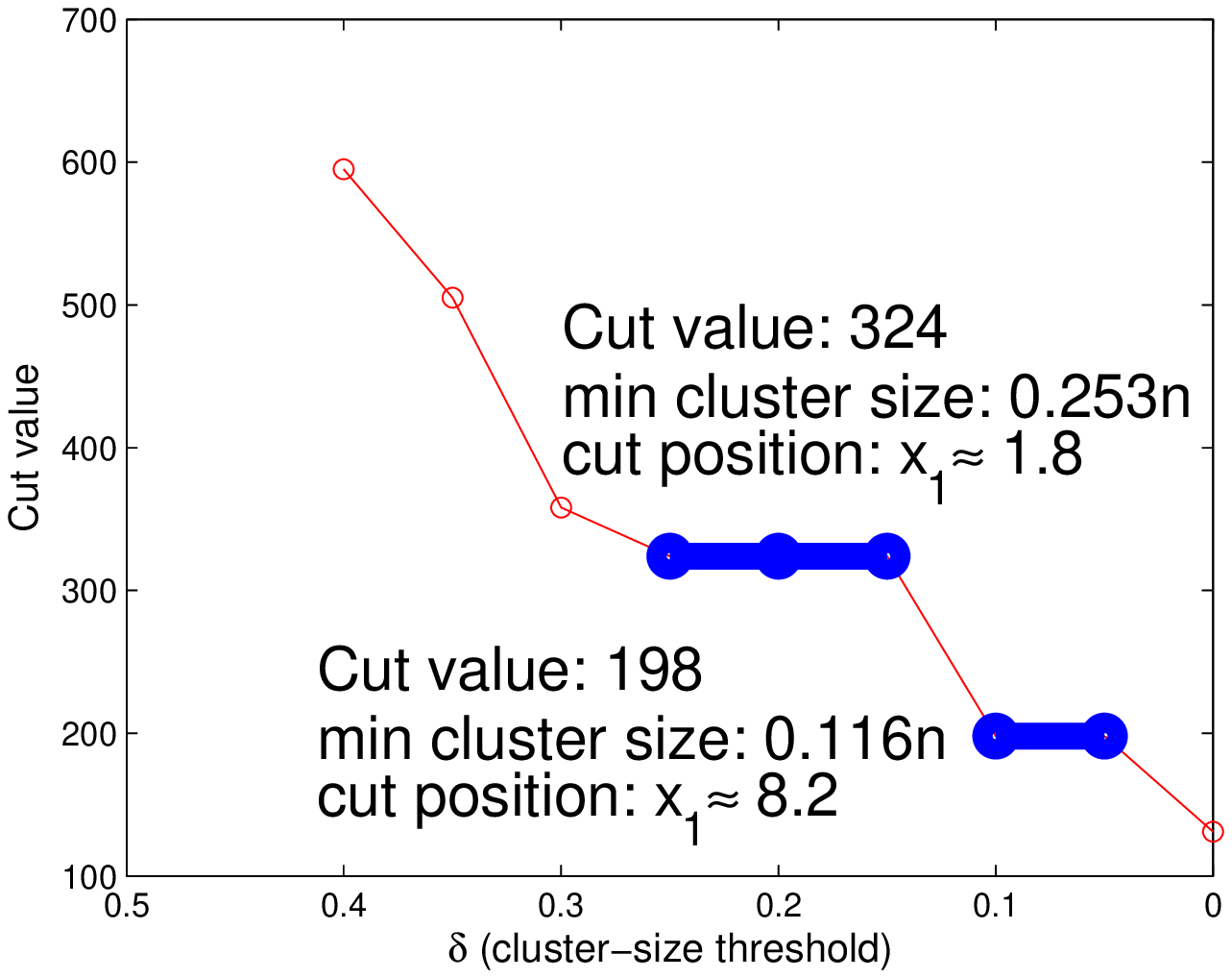}
\makebox[4.5cm]{(b) Cut value vs. $\delta$}
\end{minipage}
\begin{minipage}[t]{.32\textwidth}
\includegraphics[width = 1\textwidth]{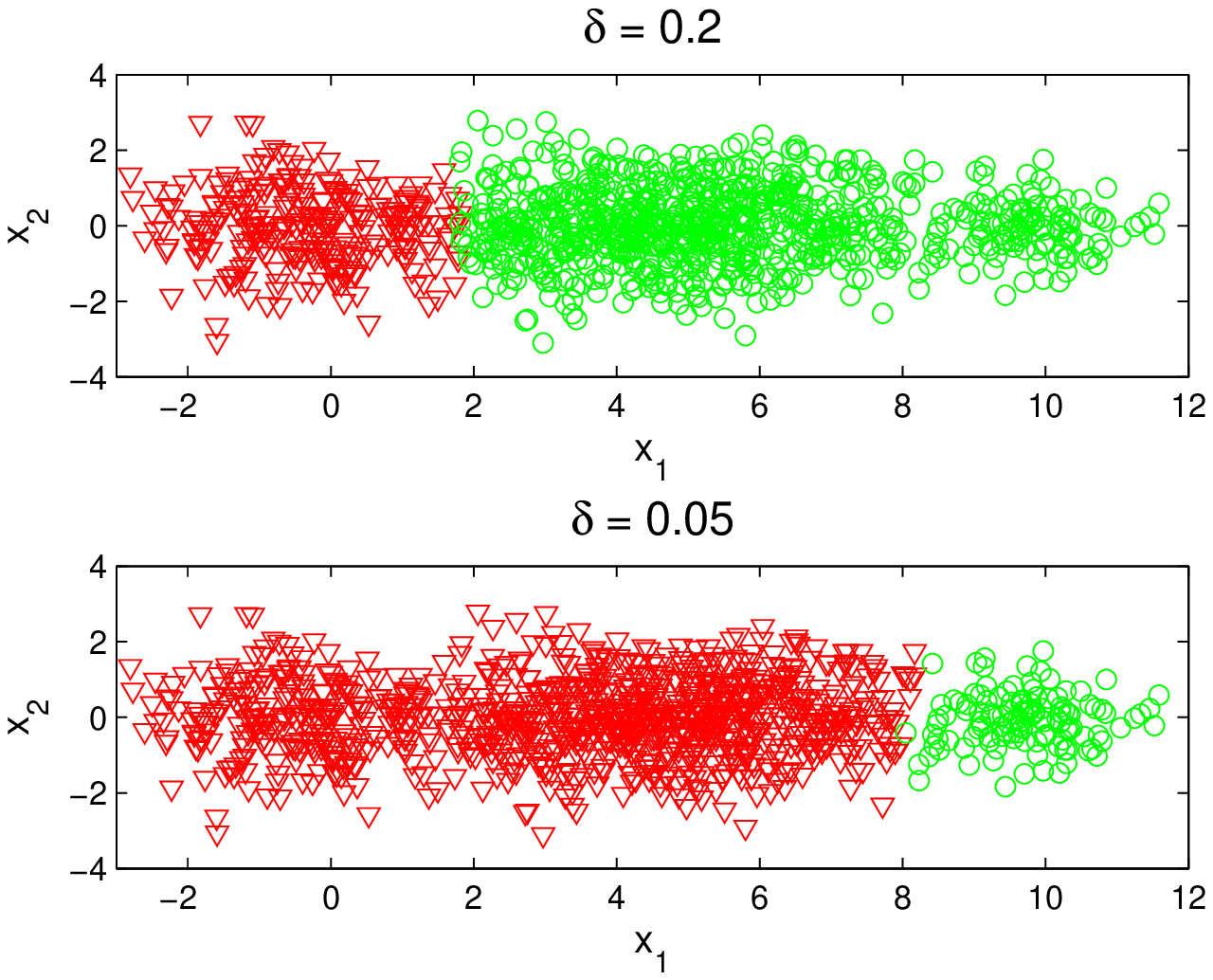}
\makebox[4.5cm]{(c) Clustering results }
\end{minipage}
\caption{ Small Cluster Detection: 2-partition SC results of 1 large and 2 small proximal gaussian components. Both valleys are at imbalanced positions. The rightmost cluster is smaller than the left, with a deeper valley. Results in (b) are from one run. As shown in (b) and (c), the left cluster is detected for a larger $\delta$, where the right smaller one is viewed as outliers. The right smaller cluster is detected by further reducing $\delta$ of (Eq.(\ref{eq:selection})).}
\label{fig:multiple_cuts}
\end{centering}
\end{figure*}

To apply our method we vary the cluster-size threshold $\delta$ in PCut. We plot PCut against $\delta$ as shown in Fig.\ref{fig:multiple_cuts}(b). As seen in Fig.\ref{fig:multiple_cuts}(b), when $\delta\geq 0.3$, the optimal cut is close to the valley. However, since the proportion of data samples in the smaller clusters is less than 30\% we see that the optimal cut is bounded away from both valleys. As $\delta$ is decreased in the range $0.25\geq\delta\geq0.15$, the optimal cut is now attained at the left valley($x_1\approx 1.8$). An interesting phenomena is that the curve flattens out in this range.
This corresponds to the fact that the cut value is minimized at this position ($x_1 = 1.8$) for any value of $\delta \in [.15,\,.25]$. This flattening out can happen only at valleys since valleys represent a ``local'' minima for the model selection step of Eq.~\ref{eq:selection} under the constraint imposed by $\delta$. Consequently, small clusters can be detected based on the flat spots. Next when we further vary $\delta$ in the region $0.1\geq\delta\geq0.05$, the best cut is attained near the right and deeper valley($x_1\approx 8.2$). Again the curve flattens out revealing another small cluster.

%%%%%%%%%%%%%%%%%%%%%%%%%%%%%%%%%%%%%%%%%%
\section{Conclusion}\label{sec:conclusion}
%%%%%%%%%%%%%%%%%%%%%%%%%%%%%%%%%%%%%%%%%%
In this paper we explain why does spectral clustering based on minimizing RCut(NCut) leads to poor clustering performance when data is imbalanced and proximal. To this end we propose the partition constraint min-cut (PCut) framework, which seeks min-cut partitions under minimum cluster size constraints. Since constrained min-cut is NP-hard, we adopt existing spectral methods (SC, GRF, GTAM) as a black-box subroutine on a parameterized family of graphs to generate candidate partitions and solve PCut on these partitions.
The parameterization of graphs is based on adaptively modulating the node degrees in varying levels to adapt to different levels of imbalanced data.
Our framework automatically selects the parameters based on PCut objective, and can be used in conjunction with other graph-based partition methods such as 1-spectral clustering, cheeger cut or sparsest cut \citep{Buhler09,Hein10,Szlam10,Arora09}.
Our idea is then justified through limit cut analysis and both synthetic and real experiments on clustering and SSL tasks.

% \acks{Acknowledgements should go at the end, before appendices and references.}
\bibliographystyle{unsrtnat}
\bibliography{RMD_bib}

\newpage

\section*{Appendix: Proofs of Theorems}
For ease of development, let $n=m_1(m_2+1)$, and divide $n$ data points into: $D=D_0 \bigcup  D_1 \bigcup ... \bigcup D_{m_1}$, where $D_0=\{x_1,...,x_{m_1}\}$, and each $D_j, j=1,...,m_1$ involves $m_2$ points. $D_j$ is used to generate the statistic $\eta$ for $u$ and $x_j\in D_0$, for $j=1,...,m_1$. $D_0$ is used to compute the rank of $u$:
\begin{equation}
    R(u) = \frac{1}{m_1}\sum_{j=1}^{m_1} \mathbb{I}_{\{ \eta(x_j;D_j)>\eta(u;D_j) \}}
\end{equation}
We provide the proof for the statistic $\eta(u)$ of the following form (here $l$ is used in place of $k_0$):
\begin{eqnarray}
  \eta(u;D_j) &=& \frac{1}{l}\sum^{l+\lfloor \frac{l}{2} \rfloor}_{i=l-\lfloor \frac{l-1}{2} \rfloor}\left( \frac{l}{i} \right)^{\frac{1}{d}}D_{(i)}(u).
\end{eqnarray}
where $D_{(i)}(u)$ denotes the distance from $u$ to its $i$-th nearest neighbor among $m_2$ points in $D_j$. Practically we can omit the weight and use the average of 1-st to $l$-st nearest neighbor distances as shown in Sec.\ref{sec:RMD_idea}.

\noindent\textbf{Proof of Theorem \ref{rank-pvalue}:}
\begin{proof}
The proof involves two steps:
\begin{itemize}
  \item[1.] The expectation of the empirical rank $\mathbb{E}\left[R(u)\right]$ is shown to converge to $p(u)$ as $n\rightarrow\infty$.
  \item[2.] The empirical rank $R(u)$ is shown to concentrate at its expectation as $n\rightarrow\infty$.
\end{itemize}
The first step is shown through Lemma \ref{lem:expectation}. For the second step, notice that the rank $R(u) = \frac{1}{m_1}\sum_{j=1}^{m_1} Y_j$, where $Y_j = \mathbb{I}_{\{ \eta(x_j;D_j)>\eta(u;D_j) \}}$ is independent across different $j$'s, and $Y_j \in [0,1]$. By Hoeffding's inequality, we have:
\begin{equation}
    \mathbb{P}\left( | R(u) - \mathbb{E}\left[R(u)\right] | > \epsilon \right) < 2\exp\left( -2m_1\epsilon^2 \right)
\end{equation}
Combining these two steps finishes the proof.
\end{proof}

\noindent\textbf{Proof of Theorem \ref{part2}:}
\begin{proof}
We want to establish the convergence result of the cut term and the balancing terms respectively, that is:
\begin{eqnarray}
    &\frac{1}{nk_n}\sqrt[d]{\frac{n}{k_n}}cut_n(S)
    \rightarrow C_d\int_S{f^{1-\frac{1}{d}}(s)\rho(s)^{1+\frac{1}{d}}ds}. \label{eq:term1}\\
    &n\frac{1}{|V^\pm|}\rightarrow
    \frac{1}{\mu(C^\pm)}. \label{eq:term2R}  \\
    &nk_n\frac{1}{vol(V^\pm)}\rightarrow
    \frac{1}{\mu(C^\pm)}. \label{eq:term2N}
    %n\frac{1}{|V^\pm|} &\rightarrow \frac{1}{2\mu(C^{\pm})} \label{eq:term3}
\end{eqnarray}
where $V^+(V^-)=\{x\in{V}: x\in{C^+}(C^-)\}$ are the discrete version of $C^+(C^-)$.

The balancing terms Eq.(\ref{eq:term2R},\ref{eq:term2N}) are obtained similarly using Chernoff bound on the sum of binomial random variables, since the number of points in $V^\pm$ is binomially distributed $Binom(n,\mu(C^\pm))$. Details can be found in \citet{Maier1}.

Eq.(\ref{eq:term1}) is established in two steps. First we can show that the LHS cut term converges to its expectation $\mathbb{E}\left(\frac{1}{nk_n}\sqrt[d]{\frac{n}{k_n}}cut_n(S)\right)$ by McDiarmid's inequality. Second we show this expectation term actually converges to the RHS of Eq.(\ref{eq:term1}). This is shown in Lemma~\ref{expectation}.

%To establish \ref{eq:term3}, the idea is that  the number of points in $V^+$ is binomially distributed $\text{Binom}(n,\mu(C^+))$. Using the Chernoff bound of binomial sum we can show that almost surely Equation \ref{eq:term3} holds true.
\end{proof}

\begin{lemma}\label{expectation}
Given the assumptions of Theorem 2,
\begin{equation}
    \mathbb{E}\left(\frac{1}{nk_n}\sqrt[d]{\frac{n}{k_n}}cut_n(S)\right)\longrightarrow C_d\int_S{f^{1-\frac{1}{d}}(s)\rho(s)^{1+\frac{1}{d}}ds}.
\end{equation}
where $C_d=\frac{2\eta_{d-1}}{(d+1)\eta_d^{1+1/d}}$.
\end{lemma}

\begin{proof}
The proof is an extension of \citet{Maier2}. The complete proof is a bit complicated, so we provide an outline here, describing the extension. More details can be found in \citet{Maier1}. The first trick is to define a cut function for a fixed point $x_i\in V^+$, whose expectation is easier to compute:
\begin{eqnarray}
cut_{x_i} = \sum_{v\in V^{-},(x_i,v)\in E}w(x_i,v).
\end{eqnarray}
Similarly, we can define $cut_{x_i}$ for $x_i\in V^-$. The expectation of $cut_{x_i}$ and  $cut_n(S)$ can be related:
\begin{eqnarray}\label{eq:expect}
\mathbb{E}(cut_n(S))=n\mathbb{E}_x(\mathbb{E}(cut_{x}))
\end{eqnarray}
Then the value of $\mathbb{E}(cut_{x_i})$ can be computed as,
\begin{equation}
    (n-1)\int_0^{\infty}{\left[\int_{B(x_i,r)\cap{C^-}}f(y)dy\right]dF_{R_{x_i}^k}(r)}.
\end{equation}
where $r$ is the distance of $x_i$ to its $k_n\rho(x_i)$-th nearest neighbor. The value of $r$ is a random variable and can be characterized by the CDF $F_{R_{x_i}^k}(r)$.
Combining equation \ref{eq:expect} we can write down the whole expected cut value
\begin{eqnarray}
  \mathbb{E}(cut_n(S)) =n\mathbb{E}_x(\mathbb{E}(cut_{x}))= n\int_{\mathbb{R}^d}f(x)\mathbb{E}(cut_{x})dx \\
   = n(n-1)\int_{\mathbb{R}^d}f(x)\left[\int_0^{\infty}{g(x,r)dF_{R_x^k}(r)}\right]dx.
\end{eqnarray}

To simplify the expression, we use $g(x,r)$ to denote
\begin{equation}
    g(x,r)=\begin{cases}
               \int_{B(x,r)\cap{C^-}}f(y)dy,  x\in{C^+} \\
               \int_{B(x,r)\cap{C^+}}f(y)dy,  x\in{C^-}.
             \end{cases}
\end{equation}

Under general assumptions, when $n$ tends to infinity, the random variable $r$ will highly concentrate around its mean $\mathbb{E}(r_x^k)$.
Furthermore, as $k_n/n\rightarrow{0}$, $\mathbb{E}(r_x^k)$ tends to zero and the speed of convergence
\begin{eqnarray}\label{eq:EkNN}
\mathbb{E}(r_x^k)\approx(k\rho(x)/((n-1)f(x)\eta_d))^{1/d}
\end{eqnarray}
So the inner integral in the cut value can be approximated by $g(x,\mathbb{E}(r_x^k))$, which implies,
\begin{equation}
    \mathbb{E}(cut_n(S))\approx{n}(n-1)\int_{\mathbb{R}^d}f(x)g(x,\mathbb{E}(r_x^k))dx.
\end{equation}

The next trick is to decompose the integral over $\mathbb{R}^d$ into two orthogonal directions, i.e., the direction along the hyperplane $S$ and its normal direction (We use $\overrightarrow{n}$ to denote the unit normal vector):
\begin{equation}
    \int_{\mathbb{R}^d}f(x)g(x,\mathbb{E}(r_x^k))dx= \\
    \int_{S}\int_{-\infty}^{+\infty}f(s+t\overrightarrow{n})g(s+t\overrightarrow{n},\mathbb{E}(r_{s+t\overrightarrow{n}}^k))dtds.
\end{equation}
When $t>\mathbb{E}(r_{s+t\overrightarrow{n}}^k)$, the integral region of $g$ will be empty: $B(x,\mathbb{E}(r_x^k))\cap{C^-}=\emptyset$. On the other hand, when $x=s+t\overrightarrow{n}$ is close to $s\in{S}$, we have the approximation $f(x)\approx{f(s)}$:
\begin{eqnarray}
  &\int_{-\infty}^{+\infty}f(s+t\overrightarrow{n})g(s+t\overrightarrow{n},\mathbb{E}(r_{s+t\overrightarrow{n}}^k))dt \\
  &\approx 2\int_{0}^{\mathbb{E}(r_{s}^k)}f(s)\left[f(s)vol\left(B(s+t\overrightarrow{n},\mathbb{E}{r_s^k})\cap{C^-}\right)\right]dt  \\
  &= 2f^2(s)\int_{0}^{\mathbb{E}(r_{s}^k)}vol\left(B(s+t\overrightarrow{n},\mathbb{E}(r_s^k))\cap{C^-}\right)dt.
\end{eqnarray}

The term $vol\left(B(s+t\overrightarrow{n},\mathbb{E}(r_s^k))\cap{C^-}\right)$ is the volume of $d$-dim spherical cap of radius $\mathbb{E}(r_s^k))$, which is at distance $t$ to the center. Through direct computation we obtain:
\begin{equation}
    \int_{0}^{\mathbb{E}(r_{s}^k)}vol\left(B(s+t\overrightarrow{n},\mathbb{E}(r_s^k))\cap{C^-}\right)dt=\mathbb{E}(r_s^k)^{d+1}\frac{\eta_{d-1}}{d+1}.
\end{equation}
Combining the above step and plugging in the approximation of $\mathbb{E}(r_s^k)$ in Eq.(\ref{eq:EkNN}), we finish the proof.
%The $k$-NN radius tends to 0, so for point $x$ and its linked neighbor $y$, $p(x)+p(y)\approx2p(x)$. Decompose the integration over $\mathbb{R}^d$ into two steps, first at point $s$ over $S$ and then along the orthogonal direction at $s$, and insert the approximation of $k$-NN radius at $s$, we can obtain the result.
\end{proof}

\begin{lemma}\label{lem:expectation}
By choosing $l$ properly, as $m_2\rightarrow\infty$, it follows that,
$$ | \mathbb{E}\left[R(u)\right] - p(u)| \longrightarrow 0$$
\end{lemma}
\begin{proof}
Take expectation with respect to $D$:
\begin{eqnarray}
\mathbb{E}_D\left[R(u)\right]
&=&\mathbb{E}_{D\backslash D_0}\left[\mathbb{E}_{D_0}\left[\frac{1}{m_1}\sum_{j=1}^{m_1}
 \mathbb{I}_{\{\eta(u;D_j)<\eta(x_j;D_j)\}}\right]\right]\\
&=&\frac{1}{m_1}\sum_{j=1}^{m_1}\mathbb{E}_{x_j}\left[
\mathbb{E}_{D_j}\left[
\mathbb{I}_{\{\eta(u;D_j)<\eta(x_j;D_j)\}}\right]\right]\\
&=&\mathbb{E}_x\left[\mathcal{P}_{D_1}\left(\eta(u;D_1)<\eta(x;D_1)\right)\right]
\end{eqnarray}
The last equality holds due to the i.i.d symmetry of $\{x_1,...,x_{m_1}\}$ and $D_1,...,D_{m_1}$. We fix both $u$ and $x$ and temporarily discarding $\mathbb{E}_{D_1}$. Let $F_x(y_1,...,y_{m_2})=\eta(x)-\eta(u)$, where $y_1,...,y_{m_2}$ are the $m_2$ points in $D_1$. It follows:
\begin{equation}
    \mathcal{P}_{D_1}\left(\eta(u)<\eta(x)\right)
    =\mathcal{P}_{D_1}\left(F_x(y_1,...,y_{m_2})>0\right)
    =\mathcal{P}_{D_1}\left(F_x-\mathbb{E}F_x>-\mathbb{E}F_x\right).
\end{equation}

To check McDiarmid's requirements, we replace $y_j$ with $y_j'$. It is easily verified that $\forall j=1,...,m_2$,
\begin{equation}\label{equ:mcdiarmid_condition}
    |F_x(y_1,...,y_{m_2})-F_x(y_1,...,y_j',...,y_{m_2})| \leq 2^{\frac{1}{d}}\frac{2C}{l} \leq \frac{4C}{l}
\end{equation}
where $C$ is the diameter of support. Notice despite the fact that $y_1,...,y_{m_2}$ are random vectors we can still apply MeDiarmid's inequality, because according to the form of $\eta$, $F_x(y_1,...,y_{m_2})$ is a function of $m_2$ i.i.d random variables $r_1,...,r_{m_2}$ where $r_i$ is the distance from $x$ to $y_i$.
Therefore if $\mathbb{E}F_x<0$, or $\mathbb{E}\eta(x)<\mathbb{E}\eta(u)$, we have by McDiarmid's inequality,
\begin{equation}
    \mathcal{P}_{D_1}\left(\eta(u)<\eta(x)\right)
    = \mathcal{P}_{D_1}\left( F_x > 0 \right)
    = \mathcal{P}_{D_1}\left( F_x-\mathbb{E}F_x>-\mathbb{E}F_x \right)
    \leq \exp\left(-\frac{(\mathbb{E}F_x)^2 l^2}{8C^2m_2}\right)
\end{equation}
Rewrite the above inequality as:
\begin{equation}\label{equ:bound_no_expectation}
    \mathbb{I}_{\{\mathbb{E}F_x>0\}}-e^{-\frac{(\mathbb{E}F_x)^2 l^2}{8C^2m_2}}
    \leq \mathcal{P}_{D_1}\left( F_x > 0 \right)
    \leq \mathbb{I}_{\{\mathbb{E}F_x>0\}}+e^{-\frac{(\mathbb{E}F_x)^2 l^2}{8C^2m_2}}
\end{equation}
It can be shown that the same inequality holds for $\mathbb{E}F_x>0$, or $\mathbb{E}\eta(x)>\mathbb{E}\eta(u)$. Now we take expectation with respect to $x$:
\begin{equation}\label{equ:bound_with_expectation}
    \mathcal{P}_x\left(\mathbb{E}F_x>0\right)-\mathbb{E}_x\left[e^{-\frac{(\mathbb{E}F_x)^2 l^2}{8C^2m_2}}\right] \leq
    \mathbb{E}\left[\mathcal{P}_{D_1}\left( F_x > 0 \right)\right] \leq \mathcal{P}_x\left(\mathbb{E}F_x>0\right)+\mathbb{E}_x\left[e^{-\frac{(\mathbb{E}F_x)^2 l^2}{8C^2m_2}}\right]
\end{equation}
Divide the support of $x$ into two parts, $\mathbb{X}_1$ and $\mathbb{X}_2$, where $\mathbb{X}_1$ contains those $x$ whose density $f(x)$ is relatively far away from $f(u)$, and $\mathbb{X}_2$ contains those $x$ whose density is close to $f(u)$. We show for $x \in \mathbb{X}_1$, the above exponential term converges to 0 and $\mathcal{P}\left(\mathbb{E}F_x>0\right) = \mathcal{P}_x\left( f(u)>f(x) \right)$, while the rest $x\in\mathbb{X}_2$ has very small measure. Let $A(x)=\left(\frac{k}{f(x) c_d m_2}\right)^{1/d}$. By Lemma \ref{lem:bound_expectation} we have:
\begin{equation}
    | \mathbb{E}\eta(x) - A(x) | \leq \gamma \left(\frac{l}{m_2}\right)^{\frac{1}{d}} A(x)
    \leq \gamma \left(\frac{l}{m_2}\right)^{\frac{1}{d}} \left(\frac{l}{f_{min}c_d m_2}\right)^{\frac{1}{d}}
    =    \left(\frac{\gamma_1}{c_d^{1/d}}\right) \left(\frac{l}{m_2}\right)^{\frac{2}{d}}
\end{equation}
where $\gamma$ denotes the big $O(\cdot)$, and $\gamma_1 = \gamma \left(\frac{1}{f_{min}}\right)^{1/d}$. Applying uniform bound we have:
\begin{equation}
    A(x)-A(u)- 2\left(\frac{\gamma_1}{c_d^{1/d}}\right) \left(\frac{l}{m_2}\right)^{\frac{2}{d}}
    \leq \mathbb{E}\left[\eta(x) - \eta(u)\right]
    \leq A(x)-A(u)+ 2\left(\frac{\gamma_1}{c_d^{1/d}}\right) \left(\frac{l}{m_2}\right)^{\frac{2}{d}}
\end{equation}
Now let $\mathbb{X}_1=\{ x:|f(x)-f(u)|\geq 3\gamma_1 d f_{min}^{\frac{d+1}{d}} \left(\frac{l}{m_2}\right)^{\frac{1}{d}} \}$. For $x\in \mathbb{X}_1$, it can be verified that $|A(x)-A(u)|\geq 3\left(\frac{\gamma_1}{c_d^{1/d}}\right) \left(\frac{l}{m_2}\right)^{\frac{2}{d}}$, or $|\mathbb{E}\left[\eta(x) - \eta(u)\right]| > \left(\frac{\gamma_1}{c_d^{1/d}}\right) \left(\frac{l}{m_2}\right)^{\frac{2}{d}}$, and $\mathbb{I}_{\{f(u)>f(x)\}}=\mathbb{I}_{\{\mathbb{E}\eta(x)>\mathbb{E}\eta(u)\}}$. For the exponential term in Equ.(\ref{equ:bound_no_expectation}) we have:
\begin{equation}
    \exp\left(-\frac{(\mathbb{E}F_x)^2 l^2}{2C^2m_2}\right)
    \leq \exp\left(-\frac{ \gamma_1^2 l^{2+\frac{4}{d}} }{ 8C^2 c_d^{\frac{2}{d}} m_2^{1+\frac{4}{d}} } \right)
\end{equation}
For $x\in \mathbb{X}_2=\{x:|f(x)-f(u)|< 3\gamma_1 d \left(\frac{l}{m_2}\right)^{\frac{1}{d}}f_{min}^{\frac{d+1}{d}} \}$, by the regularity assumption, we have $\mathcal{P}(\mathbb{X}_2)<3M\gamma_1 d \left(\frac{l}{m_2}\right)^{\frac{1}{d}}f_{min}^{\frac{d+1}{d}}$. Combining the two cases into Equ.(\ref{equ:bound_with_expectation}) we have for upper bound:
\begin{eqnarray}
  \mathbb{E}_D\left[R(u)\right]
  &=& \mathbb{E}_x\left[\mathcal{P}_{D_1}\left(\eta(u)<\eta(x)\right)\right] \\
  &=& \int_{\mathbb{X}_1}\mathcal{P}_{D_1}\left(\eta(u)<\eta(x)\right)f(x)dx +  \int_{\mathbb{X}_2}\mathcal{P}_{D_1}\left(\eta(u)<\eta(x)\right)f(x)dx \\
  &\leq& \left( \mathcal{P}_x\left(f(u)>f(x)\right) + \exp\left(-\frac{ \gamma_1^2 l^{2+\frac{4}{d}} }{ 8C^2 c_d^{\frac{1}{d}} m_2^{1+\frac{4}{d}} } \right) \right)\mathcal{P}(x\in \mathbb{X}_1) + \mathcal{P}(x\in \mathbb{X}_2) \\
  &\leq&  \mathcal{P}_x\left(f(u)>f(x)\right) + \exp\left(-\frac{ \gamma_1^2 l^{2+\frac{4}{d}} }{ 8C^2 c_d^{\frac{1}{d}} m_2^{1+\frac{4}{d}} } \right) + 3M\gamma_1 d f_{min}^{\frac{d+1}{d}} \left(\frac{l}{m_2}\right)^{\frac{1}{d}}
\end{eqnarray}
Let $l=m_2^\alpha$ such that $\frac{d+4}{2d+4}<\alpha<1$, and the latter two terms will converge to 0 as $m_2 \rightarrow \infty$. Similar lines hold for the lower bound. The proof is finished.
\end{proof}

\begin{lemma}\label{lem:bound_expectation}
Let $A(x)=\left(\frac{l}{m c_d f(x)}\right)^{1/d}$, $\lambda_1 = \frac{\lambda}{f_{min}}\left(\frac{1.5}{c_d f_{min}}\right)^{1/d}$. By choosing $l$ appropriately, the expectation of $l$-NN distance $\mathbb{E}D_{(l)}(x)$ among $m$ points satisfies:
\begin{equation}
    | \mathbb{E}D_{(l)}(x) - A(x) | = O\left( A(x) \lambda_1 \left(\frac{l}{m}\right)^{1/d} \right)
\end{equation}
\end{lemma}

\begin{proof}
Denote $r(x,\alpha)=\min\{r:\mathcal{P}\left(B(x,r)\right)\geq \alpha\}$. Let $\delta_m \rightarrow 0$ as $m \rightarrow \infty$, and $0<\delta_{m}<1/2$.
Let $U\sim Bin(m,(1+\delta_m)\frac{l}{m})$ be a binomial random variable, with $\mathbb{E}U = (1+\delta_{m})l$. We have:
\begin{eqnarray}
  \mathcal{P}\left(D_{(l)}(x)>r(x,(1+\delta_m)\frac{l}{m})\right)
  &=& \mathcal{P}\left(U<l\right) \\
  &=& \mathcal{P}\left(U<\left(1-\frac{\delta_m}{1+\delta_m}\right)(1+\delta_m)l\right) \\
  &\leq& \exp\left(-\frac{\delta_m^2 l}{2(1+\delta_m)}\right)
\end{eqnarray}
The last inequality holds from Chernoff's bound. Abbreviate $r_1 = r(x,(1+\delta_m)\frac{l}{m})$, and $\mathbb{E}D_{(l)}(x)$ can be bounded as:
\begin{eqnarray}
  \mathbb{E}D_{(l)}(x)
  &\leq& r_1\left[1-\mathcal{P}\left(D_{(l)}(x)>r_1\right)\right] + C\mathcal{P}\left(D_{(l)}(x)>r_1\right)  \\
  &\leq& r_1 + C \exp\left(-\frac{\delta_m^2 l}{2(1+\delta_m)}\right)
\end{eqnarray}
where $C$ is the diameter of support. Similarly we can show the lower bound:
\begin{equation}
    \mathbb{E}D_{(l)}(x) \geq r(x,(1-\delta_m)\frac{l}{m}) - C \exp\left(-\frac{\delta_m^2 l}{2(1-\delta_m)}\right)
\end{equation}
Consider the upper bound. We relate $r_1$ with $A(x)$. Notice $\mathcal{P}\left(B(x,r_1)\right)=(1+\delta_m)\frac{l}{m} \geq c_d r_1^d f_{min}$, so a fixed but loose upper bound is $r_1 \leq \left(\frac{(1+\delta_m)l}{c_d f_{min} m}\right)^{1/d} = r_{max}$. Assume $l/m$ is sufficiently small so that $r_1$ is sufficiently small. By the smoothness condition, the density within $B(x,r_1)$ is lower-bounded by $f(x)-\lambda r_1$, so we have:
\begin{eqnarray}
  \mathcal{P}\left(B(x,r_1)\right) &=& (1+\delta_m)\frac{l}{m} \\
  &\geq& c_d r_1^d \left( f(x)-\lambda r_1 \right)\\
  &=& c_d r_1^d f(x)\left( 1-\frac{\lambda}{f(x)}r_1 \right) \\
  &\geq& c_d r_1^d f(x)\left( 1-\frac{\lambda}{f_{min}}r_{max} \right)
\end{eqnarray}
That is:
\begin{equation}
    r_1 \leq A(x)\left( \frac{1+\delta_m}{1-\frac{\lambda}{f_{min}}r_{max}} \right)^{1/d}
\end{equation}
Insert the expression of $r_{max}$ and set $\lambda_1 = \frac{\lambda}{f_{min}}\left(\frac{1.5}{c_d f_{min}}\right)^{1/d}$, we have:
\begin{eqnarray}
  \mathbb{E}D_{(l)}(x)-A(x) &\leq& A(x)\left( \left(\frac{1+\delta_m}{1-\lambda_1 \left(\frac{l}{m}\right)^{1/d}}\right)^{1/d} -1 \right) + C \exp\left(-\frac{\delta_m^2 l}{2(1+\delta_m)}\right) \\
  &\leq& A(x)\left( \frac{1+\delta_m}{1-\lambda_1 \left(\frac{l}{m}\right)^{1/d}}-1 \right) + C \exp\left(-\frac{\delta_m^2 l}{2(1+\delta_m)}\right) \\
  &=& A(x)\frac{\delta_m + \lambda_1 \left(\frac{l}{m}\right)^{1/d}}{1-\lambda_1\left(\frac{l}{m}\right)^{1/d}} + C \exp\left(-\frac{\delta_m^2 l}{2(1+\delta_m)}\right) \\
  &=& O\left( A(x) \lambda_1 \left(\frac{l}{m}\right)^{1/d} \right)
\end{eqnarray}
The last equality holds if we choose $l=m^{\frac{3d+8}{4d+8}}$ and $\delta_m=m^{-\frac{1}{4}}$. Similar lines follow for the lower bound. Combine these two parts and the proof is finished.

\end{proof}

\end{document}